\documentclass[sigconf]{aamas}  

\usepackage{balance} 

\usepackage{amssymb}
\usepackage{amsmath}
\usepackage{autobreak}

\usepackage{mathtools}
\usepackage{float}

\usepackage{tikz}
\usetikzlibrary{trees}

\usepackage[inline, shortlabels]{enumitem}
\setlist{leftmargin=*, nosep}

\usepackage{macros_KC_TO}

\setcopyright{none}
\renewcommand\footnotetextcopyrightpermission[1]{}

\begin{document}

\title[Resolving Conflicts in Clinical Guidelines using Argumentation]{Resolving Conflicts in Clinical Guidelines using Argumentation}


%
\author{Kristijonas \v Cyras}
\orcid{0000-0002-4353-8121}
\affiliation{
 \institution{Imperial College London}
  \streetaddress{South Kensington Campus}
  \city{London} 
  \country{United Kingdom}
  \postcode{SW7 2AZ}
}
\email{k.cyras@imperial.ac.uk}
\author{Tiago Oliveira}
\affiliation{
  \institution{National Institute of Informatics}
  \streetaddress{2-1-2 Hitotsubashi, Chiyoda-ku}
  \city{Tokyo} 
  \country{Japan}
  \postcode{101-8430}
}
\email{toliveira@nii.ac.jp}

\begin{abstract}
Automatically reasoning with conflicting generic clinical guidelines is a burning issue in patient-centric medical reasoning where patient-specific conditions and goals need to be taken into account. 
It is even more challenging in the presence of preferences such as patient's wishes and clinician's priorities over goals. 
We advance a structured argumentation formalism for reasoning with conflicting clinical guidelines, patient-specific information and preferences.
Our formalism integrates assumption-based reasoning and goal-driven selection among reasoning outcomes.  
Specifically, we assume applicability of guideline recommendations concerning the generic goal of patient well-being, 
resolve conflicts among recommendations using patient's conditions and preferences, 
and then consider prioritised patient-centered goals to yield non-conflicting, goal-maximising and preference-respecting recommendations.
We rely on the state-of-the-art Transition-based Medical Recommendation model for representing guideline recommendations and augment it with context 
given by the patient's conditions, goals, as well as preferences over recommendations and goals. 
We establish desirable properties of our approach in terms of sensitivity to recommendation conflicts and patient context.
\end{abstract}


\maketitle

\section{Introduction}
\label{sec:intro}


Medical reasoning involves careful deliberation about the condition of a patient and possible treatments. 
Clinical guidelines provide best practice recommendations for achieving patient well-being given a disease and describe management of a generic patient, recommending multiple options to choose among, given a concrete patient and their context. 
When managing multiple health conditions (multimorbidities), guidelines need to be merged, 
whence multiple interactions must be considered, as they influence the evolution of the patient \cite{Grace.et.al:2013,Fraccaro.et.al:2015}. 
In particular, the recommended actions may be inapplicable, conflicting, overlapping, and so forth.
It is hard for clinicians to follow the best practices in the presence of conflicts among guidelines. 
In such settings, knowledge representation methods come handy, 
particularly in representation of guidelines and their interactions. 

Transition-based Medical Recommendation model (TMR) \cite{Zamborlini.et.al:2017} is a state-of-the-art \cite{Riano:Ortega:2017} development in representation of clinical guideline recommendations. 
TMR identifies components and relations typically present in multimorbidity situations, 
such as clinical care actions and their effects on physical properties. 
To capture interactions when merging guidelines, In \cite{Zamborlini.et.al:2017} a mechanism is advanced to identify relationships among multiple recommendations, such as contradiction, repetition, alternative. 
TMR offers a comprehensive template for clinical guidelines and their interactions.
Yet, TMR does not afford a method for representing patient-specific information. 
More importantly, TMR does not afford \emph{reasoning} mechanisms to determine which recommendations to follow given a patient. 

In general, whereas representation of clinical guidelines is a well-advanced area, 
automated reasoning with those representations, especially in the presence of conflicts, is a limiting factor \cite{Peleg:2013,Fraccaro.et.al:2015,Riano:Ortega:2017}.
An additional hurdle is taking into account the context of the patient, 
pertaining to patient-specific conditions, patient-centric goals and preferences from various parties involved \cite{Peleg:2013,Sacchi.et.al:2015,Vermunt.et.al:2017}. 
For instance, Ariadne principles \cite{Muth.et.al:2014} show the importance and difficulty of integrating interaction assessment, individual management and patient's and/or clinician's preferences in multimorbidity setting.
We here advance a framework to address the above mentioned issues by applying an argumentation-based method 
to allow an autonomous agent to reason with conflicting clinical guidelines in the context of patient information, goals and various preferences.

Generally speaking, argumentation allows to reason with incomplete and conflicting information in a way that aims to emulate human reasoning. 
Argumentation is used for modelling reasoning of autonomous agents in multi-agent systems, 
see e.g.\ \cite{Parsons:Sierra:Jennings:1998,Kakas:Moraitis:2003,Rahwan:Simari:2009,Bench-Capon:Atkinson:McBurney:2012},
and has been extensively applied in the medical domain, 
see e.g.\ \cite{Fox.et.al:2006,Tolchinsky:Cortes:Modgil:Caballero:Lopez-Navidad:2006,Hunter:Williams:2012,Longo:2016,Oliveira:et.al:2018,Cyras:Oliveira:2018}.
In medical reasoning particularly, ``argumentation is appealing as it allows for important conflicts to be highlighted and analysed and unimportant conflicts to be suppressed.'' \cite{Atkinson.et.al:2017} 
We propose to use the structured argumentation (see e.g.~\cite{Tutorials:2014}) formalism 
\emph{Assumption-Based Argumentation with Preferences} (\emph{\abap}) \cite{Bondarenko:Dung:Kowalski:Toni:1997,Cyras:Toni:2016-KR} 
for automating patient-centric reasoning with conflicting guideline recommendations and preferences. 
This choice is motivated by the simplicity of knowledge representation in \abap\ and its dealing with preferences differently than other formalisms: 
it reverses attacks due to preferences and, importantly, in doing so preserves conflict-freeness of sets of assumptions; 
these aspects allow for a leaner representation while yielding desirable properties, as demonstrated further ahead.
\abap\ also has known complexity results \cite{Dimopoulos:Nebel:Toni:2002,Lehtonen.et.al:2019-AAAI} and a working implementation \cite{Bao:Cyras:Toni:2017}.

\abap\ is deployed to use TMR for representation of recommendations and interactions via rules and arguable elements from which arguments (as deductions) are constructed. 
\abap\ augments this representation with patient-specific information and uses extension-based argumentation semantics for reasoning. 
It thus provides an assumption-driven reasoning method whereby assumed applicability of recommendations is argued about using patient's conditions.
At the same time, \abap\ deals with preferences over actions (equivalently, recommendations) as specified by e.g.~the patient. 
We establish that the resulting recommendations (corresponding to extensions of \abap\ frameworks) are non-conflicting. 
We also augment \abap\ to form \abapg, incorporating a goal-driven reasoning mechanism to determine the best non-conflicting recommendations given patient-centric goals ordered by importance. 
This allows our approach to meet Ariadne principles. 
We also illustrate our approach with a case study from \cite{Zamborlini.et.al:2017} and obtain arguably desirable outcomes when reasoning with conflicting recommendations in the context of a patient. 

The paper is structured thus. 
In Section \ref{sec:principles} we consider desiderata for our approach in terms of patient management principles from medical literature. 
We then in Section \ref{sec:setting} describe the problem of reasoning with interacting recommendations in the context of a patient. 
In Section \ref{sec:aba+} we propose to use \abap\ and its development \abapg\ for assumption-based patient-centric reasoning with recommendations, goals and preferences. 
We discuss related work in Section \ref{sec:related} 
and finish with conclusions and future work in Section \ref{sec:future}.

\section{Principles of Patient Management}
\label{sec:principles}


We situate our work in the context of principles of patient management in multimorbidity setting. 
Several works mention various principles for patient management \cite{Peleg:2013,Grace.et.al:2013,Fraccaro.et.al:2015,Sacchi.et.al:2015,Vermunt.et.al:2017}, but do so in a loose and fragmented manner. 
Hence, it is difficult to acquire a comprehensive understanding of what these principles should be. 
\cite{Muth.et.al:2014} is among the very few works with a comprehensive enumeration and description of such principles, called \textbf{Ariadne principles}. 
Our interpretation of them is as follows.

\begin{enumerate}[1.]
    \item \textbf{Interaction assessment:} recommendation interactions and respective effects are identified and resolved. 
    In contrast to patients with a single disease, when managing patients with multimorbidities, 
    a variety of potential interactions between diseases and treatments may occur and worsen the course of the disease(s).
    
    \item \textbf{Prioritisation and patient preferences:} to guide the reasoning, priorities among goals are established while respecting the patient's preferences and state. 
    These priorities and preferences are used to consolidate heavy treatment burdens and competing treatment goals. 
    Treatment  goals are expressed in terms of symptom relief, disease prevention, avoidance of undesired outcomes, and preservation or improvement of life expectancy and quality.
    
    \item \textbf{Individualised management:} a treatment plan as a set of recommendations is devised in accordance with the patient's state, preferences and the prioritised goals. 
    This plan should provide non-conflicting recommendations for the given patient.
\end{enumerate}

Ariadne principles do not provide specific methods to solve recommendation conflicts or to elicit preferences, but rather state which dimensions should be accounted for while reasoning. 
Regarding treatment goals, information about the impact of treatments on life expectancy and quality of life may not (or is often not) available. 
Therefore, limiting treatment goals to symptom relief, disease prevention, avoidance of undesired outcomes---i.e., 
effects brought about (or not) by treatments---seems to be the most practical choice in these circumstances.
Furthermore, the patient's preferences over actions and the clinician's priorities over goals should result from a discussion between the patient and the physician, 
and need to be taken into account when devising a treament plan for the patient. 

For an autonomous agent to reason with conflicting medical recommendation within the TMR model, we need to establishing foundations for that reasoning in a setting of patient management. 
TMR provides an expressive representation template for clinical guideline recommendations in discordant multimorbidity, respective interaction types, and several measures such as causation belief, deontic strength, and evidence level. 
However, it does not demonstrate the aggregation of these elements in reasoning to produce treatment solutions for specific patients. 
Devising such solutions is no trivial task  not only due to the complexity brought about by the number of existing health conditions and recommendations but also due to the overall under-specification of how decisions should be made in this setting. 
In this work, we aim to meet Ariadne principles by adequately handling conflicting guideline recommendations afforded by the TMR model, while taking into account the context that includes patient-specific conditions, goals and preferences. 

\section{Problem Setting}
\label{sec:setting}


We here situate 
the problem of reasoning with interacting clinical guideline recommendations in the context of a patient. 
We first review the Transition-based Medical Recommendation (TMR) model 
and interactions among recommendations. 
We then discuss the context of a patient.

\subsection{TMR Model}
\label{subsec:tmr}

We first give the TMR model together with guideline recommendation interaction representation. 
They will be used to construct \abap\ frameworks for reasoning with guidelines.
(As in \cite{Zamborlini.et.al:2017}, we assume that a set of guidelines is merged into a single guideline so that recommendations are delivered by the same larger guideline.)

\subsubsection{Recommendations}
\label{subsubsec:recom}

\begin{figure*}[!ht]
\centering
\includegraphics[width=\textwidth]{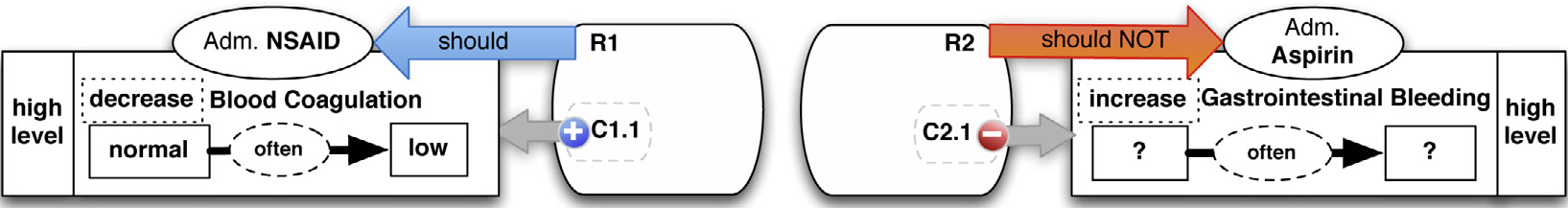}
\caption{TMR representation schema instantiated with recommendations $\rec[1]$ and $\rec[2]$ \cite[p.~83, Figure 2]{Zamborlini.et.al:2017}.}
\label{fig:schema}
\end{figure*}

Figure \ref{fig:schema} depicts an instance of a graphical schema for representing recommendations in TMR. 
(Recommendation concerning NSAID is taken from a diabetes guideline, and recommendation concerning Aspirin is taken from an osteoarthritis guideline.)
It consists of the following components.\footnote{
The formal description of recommendations, 
with components as functions/relations, is long, cumbersome, and unnecessary for the purposes of this paper. 
Instead, we give an intermediate representation, 
following the alternative formal description (and visualisation) in \cite{Zamborlini.et.al:2017} of TMR instances, 
faithful to the original but omitting certain aspects 
(as indicated below), which carries the necessary aspects required in this work.}

\begin{enumerate}[label=\bf(\roman*)]
	\item \textbf{Name}, e.g.~$\rec[1]$, $\rec[2]$, at the top of a rounded box.

(We write $\rec[k]$ instead of $\rec k$.) 

Henceforth, we refer to a recommendation by its name.

	\item A unique associated \textbf{action} $\act$, e.g.~$\aspirin$, $\nsaid$ 
(where \textit{Adm.} stands for Administer).

	\item \textbf{Deontic strength}, which we denote by $\ds$, is indicated by a thick labelled arrow and ``reflects a degree of obligatoriness expected for that recommendation'' \cite[p.~82]{Zamborlini.et.al:2017}. 
It takes values in $\left[-1, 1\right]$: 
if $\ds \geqslant 0$, then $\rec$ recommends to perform the action; 
if $\ds < 0$, then $\rec$ recommends to avoid the action. 
To discretise $\ds$, the qualitative landmarks 
\emph{must}, \emph{should}, \emph{may}, \emph{should not}, \emph{must not} 
corresponding to $1$, $0.5$, $0$, $-0.5$, $-1$, respectively, are used.
E.g., the deontic strengths of $\rec[1]$ and $\rec[2]$ in Figure \ref{fig:schema} are $\ds[1] = 0.5 = \should$ and $\ds[2] = -0.5 = \shouldnot$, respectively. 

	\item \textbf{Properties} that the associated action affects, e.g.~$\coagulation$, $\bleeding$.
(If clear from the context, we abbreviate words as follows: 
e.g.~\textit{Coag.} and \textit{Gastro.} abbreviate \textit{Coagulation} and \textit{Gastrointestinal}, respectively.) 

In general, an action can affect more than one property $\prop$. 

	\item \textbf{Effects} of the actions, e.g.~$\decrease$, $\increase$. 

An action $\act$ has one effect $\eff$ on the property $\prop$ it affects. 

	\item \textbf{Initial} and \textbf{final} values of the property that an action affects. 
For instance, \nsaid\ leads to a \decrease\ in \coagulation\ from the initial value $\normal$ to the final value $\low$. 
Otherwise, $?$ represents \emph{indeterminate} value. 

In this paper we \emph{will not} make use of, but mention for completeness, two quantitative values associated with an effect:
\emph{causation probability} -- e.g.~$\often$ -- representing the likelihood of the action bringing the effect about; 
and \emph{belief strength} -- e.g.~$\highlevel$ -- representing the level of evidence regarding bringing the effect about. 

	\item \textbf{Contributions} of the recommendation to the overall goals in the context of a guideline, e.g.~$+C1.1$, $-C2.1$, indicated below the recommendation name. 

A recommendation can have multiple contributions, 
each carrying an \emph{identifier}, e.g.~$C1.1$, $C2.1$, 
and \emph{valued} in $[-1, 1]$ 
(indicating importance of achieving/avoiding the corresponding effect), 
discretised with signs $+$, $-$ and no sign, representing values greater than, less than and equal to $0$, respectively.
\end{enumerate}

An instance of TMR concerns a generic patient. 
In order to apply recommendations, one needs to consider specific patient \emph{conditions}, 
pertaining to properties and the initial values of the effects that actions have on properties. 
For instance, a patient can have conditions $\normal~\coagulation$ or $\bleeding$. 
When using argumentation frameworks to reason with guidelines in Section \ref{sec:aba+}, 
patient conditions will come as information additional to TMR instances. 
With the following intermediate representation of TMR instances we ensure that recommendations as well as patient-specific conditions will be representable in argumentation frameworks. 

\begin{definition}
\label{defn:recom}
A \textbf{recommendation} is a tuple $\recom$ consisting of the following components:
\begin{enumerate}[label=\bf(\roman*)]
\item name $\rec$, 
\item action $\act$, 
\item deontic strength $\ds$, 
\item properties $\props = \langle \prop[1], \ldots, \prop[n] \rangle$ affected, for $n \geqslant 1$, 
\item effects $\effs = \langle \eff[1], \ldots, \eff[n] \rangle$, 
\item initial values $\vals = \langle \val[1], \ldots, \val[n] \rangle$ of effects on properties, 
\item contribution values $\contributions = \langle \contribution[1], \ldots, \contribution[n] \rangle$.
\end{enumerate}
\end{definition}

We identify any recommendation with its name $\rec$ and with an abuse of notation may write $\rec =$ $\recom$. 
We use $\recs$ to denote a fixed but otherwise arbitrary set of recommendations, 
unless specified otherwise.

\begin{example}
\label{ex:recom}
Recommendations $\rec[1] =$ $( \rec[1], \nsaid, \should,
\break \langle \coagulation \rangle, \langle \decrease \rangle, \langle \normal \rangle, \langle + \rangle )$ and
$\rec[2] =$ \linebreak $( \rec[2], \aspirin,  \shouldnot, \langle \bleeding \rangle, \langle \increase \rangle, \langle ? \rangle, \langle - \rangle )$ 
are illustrated in Figure \ref{fig:schema}. 
So $\recs = \{ \rec[1], \rec[2] \}$. 
\end{example}

\subsubsection{Interactions}
\label{subsubsec:int}

Using TMR, \citeauthor{Zamborlini.et.al:2017} identify \emph{interactions} among recommendations. 
Intuitively, interactions record the relationships between different recommendations, for instance, a \emph{contradiction} relationship in case one recommendation urges to avoid the action suggested by another recommendation. 
Several types of interactions, such as contradiction, repetition, alternative, can be identified. 
We focus on the contradiction interaction in this paper, because it relates recommendations  in direct conflict that can be naturally resolved by means of argumentation. 

Contradiction interactions can be represented as triples $(\rec, \rec', \ms)$ with recommendations $\rec$ and $\rec'$, 
and the interaction's \emph{modal strength} $\ms$, which reflects the conclusiveness of the interaction. 
The interaction's modal strength can take two values, denoted by $\Box$ and $\Diamond$, 
where $\Box$ means `the interaction will certainly occur if the related recommendations are prescribed' \cite{Zamborlini.et.al:2017}
and $\Diamond$ means `the interaction is uncertain to happen'. 
Formally, we define:

\begin{definition}
\label{defn:int}
A \textbf{contradiction interaction} between recommendations $\rec, \rec' \in \recs$ is a tuple $(\rec, \rec', \ms)$, 
where $\ms \in \{ \Box, \Diamond \}$ is the \emph{modal strength} of the interaction. 
\end{definition}

$\ints$ denotes the set of all contradiction interactions given $\recs$.

\begin{example}
\label{ex:int}
The recommendations $\rec[1]$ and $\rec[2]$ from Example \ref{ex:recom} are in a contradiction interaction, 
as they recommend opposite actions.\footnote{Note well that a hierarchy of actions is assumed in \cite[p.~79]{Zamborlini.et.al:2017} to obtain interactions. 
For instance, the action to administer NSAID subsumes both actions to administer Aspirin and Ibuprofen. 
This hierarchy is not important for our purposes.} 
We assume that $\ints = \{ (\rec[1], \rec[2], \Box) \}$.
\end{example}

The possibility to identify interactions gives rise to the following notion of \emph{contradiction-free} sets of recommendations. 

\begin{definition}
\label{defn:contradiction-free}
A set $\recs' \subseteq \recs$ is \textbf{contradiction-free} iff there is no contradiction interaction $\int \in \ints$ with $\rec[i], \rec[j] \in \recs'$. 
\end{definition}

Intuitively, contradiction-free sets of recommendations consist of recommendations that are safe to follow without the risk of performing incompatible actions. 

\begin{example}
\label{ex:contradiction-free}
The set $\recs = \{ \rec[1], \rec[2] \}$ from Example \ref{ex:recom} is not contra\-diction-free, 
for $(\rec[1], \rec[2], \Box) \in \ints$, as in Example \ref{ex:int}. 
Clearly, $\{ \rec[1] \}$ and $\{ \rec[2] \}$ themselves are contradiction-free.
\end{example}

Our representation of recommendations and interactions as afforded by the TMR model will contribute to our approach meeting the $1^{st}$ and the $3^{rd}$ Ariadne principles as laid down in Section \ref{sec:principles}.

\subsection{Context}
\label{subsec:context}

Recommendations $\recs$ and interactions $\ints$ amount only to representation of guidelines, but not reasoning with them. 
In particular, they give a patient-agnostic representation, while the reasoning happens with patient-specific information. 

\begin{example}
\label{ex:setting}
Consider $\recs = \{ \rec[1], \rec[2] \}$ and $\ints = \{ (\rec[1], \rec[2], \Box) \}$ as in Examples \ref{ex:recom} and \ref{ex:int}. 
Intuitively, for a generic patient, NSAID -- e.g.\ Aspirin -- should be administered. 
If, however, the patient exhibits \bleeding, then \rec[1] and \rec[2] are in conflict and there are arguments for both administering and not administering Aspirin. 
\end{example}

The patient information can be understood as the \emph{context} in which reasoning happens (see e.g.~\cite{Sacchi.et.al:2015}).
To resolve the conflict in Example \ref{ex:setting}, 
one could administer a different NSAID, such as Ibuprofen. 
However, in more complicated situations such alternatives may not be available. 
In those situations, \emph{preferences} may be a part of the context that help to resolve the conflicts argumentatively. 

\begin{example}
\label{ex:preferences}
Continuing Example \ref{ex:setting}, suppose that only Aspirin is available. 
The patient may insist that medication should be given to them, thus preferring taking Aspirin over not taking it, whence only \rec[1] should be followed. 
On the other hand, if the patient expresses no preferences, the clinician's priorities may come into play. 
For instance, the clinician may deem not increasing the risk of gastrointestinal bleeding more important than decreasing blood coagulation, whence only \rec[2] would be followed. 
\end{example}

Thus, the context includes not only the patient's state, but also various preferences. 
For instance:
\begin{enumerate*}[a)]
\item the patient may prefer one course of action over another;
\item the clinician may prioritise treatments in accordance with patient-centric goals and their importance.
\end{enumerate*}
The TMR model however does not afford representation of such preferences, just as it does not afford representation of patient-specific conditions. 
One of our tasks is to augment the representation of recommendations and interactions with the context of a patient so as to enable patient-centric reasoning with clinical guidelines. 
For this purpose, we define the context pertaining to patient information with respect to recommendations as follows.

\begin{definition}
\label{defn:context}
The \textbf{context} (of a fixed but otherwise arbitrary patient) is a tuple $\context$ with: 
the patient's state $\pstates$, the patient-centric goals $\goals$, the (patient's) preferences $\leqslant$ over actions, the (clinician's) priorities $\preccurlyeq$ over goals.\footnote{Following Ariadne principles, we distinguish between preferences over actions and priorities over goals for ease of reference.}
\end{definition}

In the rest of the paper we assume that a context is compatible with given recommendations in the following sense:
the patient's state $\pstates$ matches some of the properties within recommendations; 
the goals $\goals$ match the (un)desired effects on those properties; 
the patient's preferences are (represented by a preorder) over the recommended actions or recommendations;
the clinician's priorities are (represented by a total preorder) over the effects on the patient's state.
We make this precise in Section \ref{subsec:reasoning}.

\begin{example}
\label{ex:context}
Building on Examples and \ref{ex:setting} and \ref{ex:preferences}, 
the context of the patient can be given by $\pstates = \{ \bleeding \}$, $\goals = \{ \decrease~\coagulation, \lpnot\increase~\bleeding \}$,\footnote{$\lpnot$ is purely syntactic, representing the desire to avoid the effect on the property brought about by the action.\label{footnote:not}} $\rec[2] < \rec[1]$,\footnote{As usual, 
the strict (asymmetric) counterpart $<$ of a preorder $\leqslant$ is given by 
$\alpha < \beta$ iff $\alpha \leqslant \beta$ and $\beta \nleqslant \alpha$, 
for any $\alpha$ and $\beta$. 
We assume this for all preorders in this paper.} and  $\decrease~\coagulation \prec \lpnot\increase~\bleeding$.
\end{example}

The elements together form a context for the application of recommendations and ground them to a particular setting. 
The context of a patient will contribute to our approach meeting the $2^{st}$ and the $3^{rd}$ Ariadne principles put forward in Section \ref{sec:principles}.

\section{Reasoning with Guidelines}
\label{sec:aba+}


We will use guideline recommendations, their interactions and contexts to construct argumentation frameworks for an agent to reason and resolve conflicts among recommendations, 
given patient-specific conditions, patient-centric goals and various preferences. 
Specifically, we will use \abap\ frameworks for assumption-based reasoning with guidelines and patient's preferences over recommendations. 
We will then augment \abap\ to \abapg\ for goal-driven reasoning with guidelines and clinician's priorities over goals.

\subsection{\texorpdfstring{\abap} ~ Background}
\label{subsec:aba+ background}

We provide the background for \abap\ following \cite{Bondarenko:Dung:Kowalski:Toni:1997,Cyras:Toni:2016-KR}.

\label{definition:ABA framework}
An \textbf{\abap\ framework} is a tuple $\abapf$, where:
\begin{itemize}[label=$\bullet$]
\item $(\LL, \R)$ is a deductive system with $\LL$ a language 
and $\R$ a set of rules of the form 
$\varphi_0 \ot \varphi_1, \ldots, \varphi_m$ with $m \geqslant 1$,
or of the form $\varphi_0 \ot \top$, 
where $\varphi_i \in \LL$ for $i \in \{ 0, \ldots, m \}$ and $\top \not\in \LL$; 
$\varphi_0$ is the \emph{head} or \emph{conclusion}, and $\varphi_1, \ldots, \varphi_m$ the \emph{body} of the rule;
$\varphi_0 \ot \top$ is said to have an empty body and called a \emph{fact};

\item $\A \subseteq \LL$ is a non-empty set of \emph{assumptions};
\item $\contrary : \A \to \LL$ is a total map: 
for $\asma \in \A$, $\contr{\asma}$ is referred to as the \emph{contrary} of $\asma$;
\item $\leqslant$ is a preorder (i.e.~reflexive and transitive order) on $\A$, called a \emph{preference relation}.
\end{itemize}

For $\asma, \asmb \in \A$, $\asma \leqslant \asmb$ means that $\asmb$ is at least as preferred as $\asma$, 
and $\asma < \asmb$ means that $\asma$ is strictly less preferred than $\asmb$. 

Throughout, we assume a fixed but otherwise arbitrary \abap\ framework $\F = \abapf$, unless else specified. 

Assumptions in \abap\ represent arguable information. 
For instance, assumptions can represent the applicability of, or an agent's willingness to follow, a recommendation. 
In such a case, preferences in \abap\ can represent the willingness to follow recommendations. 

We next give notions of arguments and attacks in \abap. 

An \textbf{argument for conclusion $\varphi \in \LL$ supported by $\asmA \subseteq \A$ and $R \subseteq \R$}, 
denoted $\asmA \vdash^R \varphi$, is a finite tree with: 
the root labelled by $\varphi$; 
leaves labelled by $\top$ or assumptions, with $\asmA$ being the set of all such assumptions; 
the children of non-leaves $\psi$ labelled by the elements of the body of some $\psi$-headed rule in $\R$, with $R$ being the set of all such rules. 
$\asmA \vdash \varphi$ abbreviates $\asmA \vdash^R \varphi$ with some $R \subseteq \R$. 

For $\asmA, \asmB \subseteq \A$, 
$\asmA$ \textbf{$<$-attacks} $\asmB$, 
denoted $\asmA \pattacks \asmB$,
iff:
\begin{enumerate}[label=\alph*), leftmargin=*]
\item either there is an argument $\asmA' \vdash \contr{\asmb}$, for some $\asmb \in \asmB$, 
supported by $\asmA' \subseteq \asmA$, and $\nexists \asma' \in \asmA'$ with $\asma' < \asmb$;
\item or there is an argument $\asmB' \vdash \contr{\asma}$, for some $\asma \in \asmA$, 
supported by $\asmB' \subseteq \asmB$, and $\exists \asmb' \in \asmB'$ with $\asmb' < \asma$.
\end{enumerate}
The intuition here is that $\asmA$ $<$-attacks $\asmB$ if
\begin{enumerate*}[a)]
\item either $\asmA$ argues contra something in $\asmB$ by means of no inferior elements (\emph{normal attack}), 
\item or $\asmB$ argues contra something in $\asmA$ but with at least one inferior element (\emph{reverse attack}).
\end{enumerate*}

If $\asmA$ does not $<$-attack $\asmB$, we may write $\asmA \npattacks \asmB$. 
Note that without preferences, an attack from one set of assumptions to another boils down to the former set deducing the contrary of some assumption in the latter set. 

We next give notions used to define \abap\ semantics. 

Let $\asmA \subseteq \A$. 
The \emph{conclusions of $\asmA$} is the set of sentences  
$\Cn(\asmA) = \{ \varphi \in \LL~:~\exists~\asmA' \vdash \varphi, \ \asmA' \subseteq \asmA \}$ 
concluded by (arguments supported by subsets of) $\asmA$. 
We say $\asmA$ is \emph{closed} if $\asmA = \Cn(\asmA) \cap \A$, i.e.~$\asmA$ contains all assumptions it concludes. 
We say \F\ is \emph{flat} if every $\asmA \subseteq \A$ is closed. 
We assume \abap\ frameworks to be flat in this paper.

Further: 
$\asmA$ is \emph{$<$-conflict-free} if $\asmA \npattacks \asmA$; 
also, $\asmA$ \emph{$<$-defends} $\asmA' \subseteq \A$ if  $\forall~\asmB \subseteq \A$ with $\asmB \pattacks \asmA'$ we have $\asmA \pattacks \asmB$; 
and $\asmA$ is \emph{$<$-admissible} if it is $<$-conflict-free and $<$-defends itself. 
We consider one \abap\ semantics: 
a set $\asmE \subseteq \A$ of assumptions is a \textbf{$<$-preferred extension} of $\F = \abapf$ if $\asmE$ is $\subseteq$-maximally $<$-admissible.


\subsection{\texorpdfstring{\abapg: \abap} ~ with Goals}
\label{subsec:abapg}

We extend \abap\ with a mechanism to distinguish among preferred extensions based on goals fulfilled. 
\citeauthor{Oliveira:et.al:2018} introduce goal seeking mechanisms in structured argumentation  to rank argument extensions according to their relative priorities \cite{Oliveira:et.al:2018}.
We import this goal-driven reasoning into \abap\ to define \abapg, and thus cover the important aspect of reasoning with patient-centric goals.

\begin{definition}
\label{defn:abapgf}
An \textbf{\abapg\ argumentation framework} is a tuple $\abapgf$, where 
$\abapf$ is an \abap\ framework
and
\begin{itemize}
	\item $\G \subseteq \LL$ is a finite set of \textbf{goals} such that $\forall~\theta \in \G$, there exists $\theta \leftarrow \varphi_1, \ldots, \varphi_m$ with $m\leqslant1$ in $\R$;

	\item $\preccurlyeq$ is a total preorder on $\G$, denoting \textbf{priorities} over goals; 
for $\theta, \chi \in \G$, $\theta \preccurlyeq \chi$ means $\chi$ is as important as $\theta$.	
\end{itemize}
\end{definition}
In what follows, $\abapgf$ is a fixed but otherwise arbitrary \abapg\ framework, unless said otherwise. 

In \abapg\ concluding goals amounts to fulfilling them. 
We hence define (preferred) \emph{goal extensions} in terms of goal-conclusions thus:

\begin{definition}
\label{defn:goal extension}
Let $\asmE$ be a $<$-preferred extension of $\abapf$. 
Then $\goals[\asmE] = \Cn(\asmE) \cap \goals$ is a \textbf{goal extension} of $\abapgf$.
\end{definition}

In other words, a goal extension consists of the goals concluded by a $<$-preferred extension. 
We use priorities over goals to rank goal extensions 
and define \abapg\ semantics:

\begin{definition}
\label{defn:abapg semantics}
Let $\mathbb{G}$ be the set of goal extensions.
The \emph{goal extension ordering} $\unlhd_{\mathbb{G}}$ over $\mathbb{G}$ is given by 
\begin{quote}
$\goals[\asmA] \unlhd_{\mathbb{G}} \goals[\asmB]$ ~ iff ~ $\exists \theta \in \goals[\asmB]\setminus\goals[\asmA]$ with $\chi \preccurlyeq \theta$ ~ $\forall \chi \in \goals[\asmA]\setminus\goals[\asmB]$.
\end{quote}

$\goals \in \mathbb{G}$ is a \textbf{top goal extension} iff $\nexists \goals' \in \mathbb{G}$ such that $\goals \lhd_{\mathbb{G}} \goals'$.
\end{definition}

Note that $\unlhd_{\mathbb{G}}$ is a total preorder, as $\preccurlyeq$ is a total preorder. 
Intuitively, $\goals[\asmA] \unlhd_{\mathbb{G}} \goals[\asmB]$ means that $\goals[\asmB]$ is at least as `good' as $\goals[\asmA]$. 
The underlying principle behind ordering goal extensions is trying to fulfill goals according to their importance. 

A top goal extension admits no strictly `better' goal extension. 
Intuitively, a $<$-preferred \abap\ extension inducing a top goal extension yields the best reasoning outcome.

This choice of ordering is motivated by the requirements of a patient management setting, 
within which priorities over goals may convey a sense of urgency and severity that must be addressed when reasoning. 
Hence, we assume that an agent should always aim to fulfill the top preferred goals, 
regardless of the goals with lower priorities. 
In general, preference aggregation is a rich and complex area of research. 
Other orderings could be applied, see e.g.\ \cite{Kaci:Patel:2014} for a comparison of various orderings, 
but we chose the above one in accordance to our interpretation of priorities over goals.


\subsection{Reasoning in \texorpdfstring{\abapg}~}
\label{subsec:reasoning}

We now introduce the representation in \abapg\ of TMR instances, interactions and context. 

\subsubsection{Intuition}
\label{subsubsec:intuition}

At a high-level: 
assumptions will represent (the defeasible applicability of) recommendations, 
the corresponding actions and their effects on properties will be modelled via rules, 
and the deontic strength will determine both whether the actions and their consequences are sought after or not. 
The context will be modelled via facts representing patient's state, 
goals matching the effects of actions, 
patient's preferences over assumptions 
and clinician's priorities over goals. 

For a step by step illustration, we use recommendations $\recs = \{ \rec[1], \rec[2] \}$ and interactions $\ints = \{ (\rec[1], \rec[2], \Box) \}$ as in Example \ref{ex:setting}.
First, $\rec[1], \rec[2] \in \A$ represent the (defeasible applicability of) recommendations. 
The following rules then represent the actions recommended (or not) by $\rec[1]$ and $\rec[2]$:
\begin{enumerate}[label={\bf\arabic*.}, series=main]
\item $\nsaid \ot \rec[1]$;
\item $\lpnot\aspirin \ot \rec[2]$.
\end{enumerate}

The following rules model the effects the actions $\nsaid$ and $\aspirin$ bring about:
\begin{enumerate}[resume*=main]
\item $\decrease~\coagulation \ot \nsaid$;
\item $\increase~\bleeding \ot \aspirin$.
\end{enumerate}

As $\rec[2]$ recommends $\lpnot\aspirin$, the following rule represents the effect to be avoided by following $\rec[2]$:
\begin{enumerate}[resume*=main]
\item $\lpnot\increase~\bleeding \ot \lpnot\aspirin$.
\end{enumerate}

Now, $\rec[1]$ and $\rec[2]$ are in contradiction with \nsaid\ and \aspirin\ recommended positively and negatively, respectively. 
Thus, $\rec[2]$ can be argued against on the basis of $\rec[1]$ and the presence of the contradiction. 
However, $\rec[1]$ can be similarly argued against on the basis of $\rec[2]$ and the presence of the contradiction, 
but \emph{only as long as} the given patient has \bleeding. 
Therefore, we have:
\begin{enumerate}[resume*=main]
\item $\contr{\rec[2]} \ot \rec[1], int_{1, 2}$;
\item $\contr{\rec[1]} \ot \rec[2], int_{1, 2}, \bleeding$.
\end{enumerate}

Here, $\contr{\rec[1]}$ and $\contr{\rec[2]}$ are the contraries of $\rec[1]$ and $\rec[2]$, respectively, 
and $int_{1, 2} \in \LL$ represents $(\rec[1], \rec[2], \ms) \in \ints$. 
These rules say that: \\
$\rec[2]$ should not be followed if 
\begin{enumerate}[(i)]
\item $\rec[1]$ is followed, and
\item $\rec[1]$ and $\rec[2]$ are in contradiction;
\end{enumerate}
$\rec[1]$ should not be followed if 
\begin{enumerate}[(i)]
\item $\rec[2]$ is followed, 
\item $\rec[1]$ and $\rec[2]$ are in contradiction, \emph{and also}
\item the condition $\bleeding$ is present.
\end{enumerate}

This is in accordance with the desirable reading of interactions as in Section \ref{sec:setting} and in \cite[p.~91]{Zamborlini.et.al:2017}.

The interaction's modal strength $\ms$ determines whether it is an assumption (i.e.~could be argued about) or a fact (i.e.~sure to happen):
\begin{enumerate}[a)]
\item if $\ms = \Box$, let $int \ot \top \in \R$;
\item if $\ms = \Diamond$, let $int \in \A$.
\end{enumerate}
In our example, $\ms = \Box$, so we have:

\begin{enumerate*}[resume*=main]
\item $int_{1, 2} \ot \top$.
\end{enumerate*}

Given context $\context$, the patient's state $\pstates$ yields properties/initial value-property pairs as facts. 
With context from Example \ref{ex:context}, $\bleeding \in \pstates$ yields:

\begin{enumerate*}[resume*=main]
\item $\bleeding \ot \top$.
\end{enumerate*}

Lastly, as in Example \ref{ex:context}, 
goals $\goals$ represent (un-)desired effects on properties, 
patient's preferences $\leqslant$ are over recommendations as assumptions 
and clinician's priorities are over goals. 

\subsubsection{Formalisation}
\label{subsubsec:formalisation}

Formally, mapping recommendations, interactions and context to \abapg\ goes as follows.

\begin{definition}
\label{defn: abapg patient framework}
Given recommendations $\recs$, interactions $\ints$ and context $\context$, 
\textbf{the \abapg\ patient framework} is defined as  $\F_{p} = \abapgf$, where:

\begin{itemize}
\item $\A = \{ \rec : \recom \in \recs \} \cup \{ int_{i, j} : \intD \in \ints \}$ consists of assumptions representing recommendations and uncertain interactions; 

\item $\R_a = \R^+_a \cup \R^-_a$ consists of rules representing actions associated to recommendations, where
	\begin{itemize}
	\item $\R^+_a = \{ \act \ot \rec~:~\rec \in \recs, ~ \ds \geqslant 0 \}$,
	\item $\R^-_a = \{ \lpnot \act \ot \rec~:~\rec \in \recs, ~ \ds < 0 \}$;\footnote{$\lpnot$ is purely syntactic (see footnote \ref{footnote:not}).}
	\end{itemize}

\item 
$\R_e = \R^+_e \cup \R^-_e$ consists of rules representing effects brought about by actions, where
	\begin{itemize}
	\item $\R^+_e = \{ \eff[k]\prop[k] \ot \act~:~\rec \in \recs, ~ \eff[k] \in \effs, ~ \prop[k] \in \props, ~ \ds \geqslant 0 \}$, 
	\item $\R^-_e = \{ \lpnot \eff[k]\prop[k] \ot \lpnot \act~:~\rec \in \recs, ~ \eff[k] \in \effs, ~ \prop[k] \in \props, \linebreak \ds < 0 \}$;
	\end{itemize}

\item 
$\R_s = \{ \val\prop \ot \top~:~\val\prop \in \pstates \}$ consists of facts representing the patient's state $\pstates$ in terms of properties and their values, 
where $\pstates \subseteq \bigcup_{\rec \in \recs} \{ \val[k]\prop[k]~:~\prop[k] \in \props, \val[k] \in \vals \}$;

\item $\R_c = \R^+_c \cup \R^-_c$ consists of rules representing (contradiction) interactions between recommendations, 
where
	\begin{itemize}
	\item $\R^+_c = \{ \contr{\rec[j]} \ot \rec[i], int_{i, j}~:~\int \in \ints, ~ \ds[i] \geqslant 0 \}$, 
	\item rules in $\R^-_c = \{ \contr{\rec[i]} \ot \rec[j], int_{i, j}, \val[k]_j\prop[k]_j~:~\int \in \ints, \linebreak~ \ds[j] < 0, ~ \prop[k]_j \in \props[j], ~ \val[k]_j \in \vals[j], ~ \contribution[k]_j \in \contributions[j], ~ \contribution[k]_j = - \}$ take into account presence of negatively affected conditions;
	\end{itemize}

\item $\R = \R_a \cup \R_e \cup \R_s \cup \R_c \cup \{ int_{i, j} \ot \top~:~\intB \in \ints \}$
consists of rules defined above and rules representing interactions that are sure to happen;

\item $\leqslant$ is a preorder over $\A$; 

\item $\G=\G^+ \cup \G^-$ satisfies
		
\begin{itemize}
\item $\G^+ \subseteq \bigcup_{\rec \in \recs} \{ \eff[k]\prop[k]~:~\prop[k] \in \props, \eff[k] \in \effs \}$, 

\item $\G^- \subseteq \bigcup_{\rec \in \recs} \{ \lpnot\eff[k]\prop[k]~:~\prop[k] \in \props, \eff[k] \in \effs \}$, 
\end{itemize}

\item $\preccurlyeq$ is a total preorder over $\G$;

\item By convention, $\LL$ and $~\contrary$ are implicit from $\A$ and $\R$ as follows:
unless $\contr{\asmx}$ appears in either $\A$ or $\R$, 
it is different from the sentences appearing in $\A$ or $\R$; 
thus, $\LL$ consists of all the sentences appearing in $\R$, $\A$ 
and $\{ \contr{\asma}~:~\asma \in \A \}$.
\end{itemize}
\end{definition}

Regarding interactions and rules in $\R_c$, 
suppose recommendations $\rec[i]$ and $\rec[j]$ are in contradiction with actions $\act[i]$ and $\act[j]$ recommended positively ($\ds[i] > 0$) and negatively ($\ds[j] < 0$), respectively. 
On the one hand, $\rec[j]$ can be argued against on the basis of $\rec[i]$ and the presence of the interaction. 
On the other hand, $\rec[i]$ can be similarly argued against on the basis of $\rec[j]$ and the presence of the interaction, 
but \emph{only as long as} a given patient will have some condition affected by $\act[j]$ that contributes negatively to the patient's well-being. 
Thus, we take into account any property $\prop \in \props[j]$ with initial value $\val \in \vals[j]$ and contribution $- = \contribution \ni \contributions[j]$. 
When the initial value $\val$ of $\prop$ is indeterminate $?$, we use only $\prop$.

\subsubsection{Properties}
\label{subsubsec:properties}

Modelling recommendations and interactions argumentatively 
allows to exploit properties of \abap\ to ensure desirable features of our approach. 
Specifically, the $<$-preferred extensions in \abapg\ patient frameworks are contradiction-free (Definition \ref{defn:contradiction-free}) as sets of recommendations 
(recall that we identify a recommendation with its name, see remark after Definition \ref{defn:recom}):

\begin{theorem}[\textbf{Interaction Theorem}]
\label{thm:cf}
For a $<$-preferred extension $\asmE$ of $\abapf$ in $\abapgf$, 
$\asmE \cap \recs$ is a contra\-diction-free set of recommendations.
\end{theorem}

\begin{proof}
Suppose $\asmE \cap \recs$ is not contradiction-free. 
Then there is $\int \in \ints$ with $\rec[i], \rec[j] \in \asmE$. 
But as $\contr{\rec[j]} \ot \rec[i], int_{i, j} \in \R$ and $int_{i, j}$ is either a fact or an $<$-unattacked assumption, 
we find $\asmE \pattacks \asmE$. 
This contradicts $<$-conflict-freeness of $\asmE$. 
\end{proof}

Thus, 
top goal extensions (induced by $<$-preferred extensions) in \abapg\ are guaranteed to yield goals achievable without the risk of performing incompatible actions. 

Another property states that if the patient expresses preferences over all recommendations, 
then the most preferred non-conflicting recommendations will be followed:

\begin{theorem}[\textbf{Preferences Theorem}]
\label{thm:pref}
Let $\leqslant$ be total over $\A \cap \recs$ and the set $\recs' = \{ \rec \in \A~:~\nexists \rec' \in \A \text{ with } \rec < \rec' \}$ of the most preferred recommendations be contradiction-free. 
Then $\recs' \subseteq \asmE$ for every $<$-preferred extension $\asmE$ of $\abapf$ in $\abapgf$. 
\end{theorem}

\begin{proof}
Any $\rec \in \recs'$ is $\leqslant$-maximal, so $\{ \rec \}$ is $<$-unattacked. 
As all $<$-attacks come from by singleton sets, 
every $<$-preferred extension contains all the $<$-unattacked sets of assumptions, including $\recs'$.
\end{proof}

Theorems \ref{thm:cf} and \ref{thm:pref} ensure that \abapg\ meets the three Ariadne principles of 
\textit{interaction assessment}, 
\textit{prioritisation and patient preferences} and 
\textit{individualised management} 
when applied to patient-centric reasoning with conflicting medical recommendations.


\subsubsection{Illustration}
\label{subsubsec:illustration}

We exemplify our formalisation with a case study from \cite{Zamborlini.et.al:2017}, 
focusing on contradiction interactions between breast cancer (BC) and hypertension (HT) guidelines, 
and using the pertinent parts of the information 
(given in full in \cite[p.~87, Figure 5, p.~90, Table 9, p.~91, Table 10]{Zamborlini.et.al:2017}). 
Our results will be in agreement with the informal discussion on the case study in \cite{Zamborlini.et.al:2017}.

\begin{example}
\label{ex:aba+}
We assume a merged BC and HT guideline with:
\begin{itemize}
\item $(\rec[8], \highexercise, \shouldnot, \langle \bloodpress \rangle, \langle \increase \rangle, \langle ? \rangle, \break \langle - \rangle)$, 
\item $(\rec[4], \exercise, \mustnot, \langle \bodytemp \rangle, \langle \increase \rangle, \langle \high \rangle, \langle - \rangle)$, 
\item $(\rec[3], \lowexercise, \should, \props[3], \effs[3], \vals[3], \contributions[3])$, 
\item $(\rec[2], \stdexercise, \should, \props[3] \cup \langle \lymph \rangle, \effs[3] \cup \langle \increase \rangle, \break \vals[3] \cup \langle \present \rangle, \contributions[3] \cup \langle - \rangle)$, where

\begin{itemize}
\item $\props[3] = \langle \fatigue, \fitness, \pain \rangle$, 
\item $\effs[3] = \langle \decrease, \decrease, \decrease \rangle$, 
\item $\vals[3] = \langle \high, \high, \high \rangle$, 
\item $\contributions[3] = \langle +, +, + \rangle$.
\end{itemize}

\end{itemize}

For instance, \rec[8] says that one \shouldnot\ do \highexercise, because it negatively contributes by increasing \bloodpress; 
\rec[3] says that one \should\ do \lowexercise, because it positively contributes to decreasing \fatigue, \fitness\ and \pain\ from \high\ values. 

Thus, $\recs = \{ \rec[2], \rec[3], \rec[4], \rec[8] \}$ 
and the interactions identified are 
$\ints = \{ (\rec[2], \rec[4], \Box), ~ (\rec[3], \rec[4], \Box), ~ (\rec[2], \rec[8], \Box) \}$.

To illustrate reasoning with patient-specific conditions, goals and preferences, 
we assume, in addition to the case study of \cite{Zamborlini.et.al:2017}, \emph{patient A}. 
Let patient A exhibit increased \bloodpress\ (indeterminate value), and in addition have \high~\bodytemp\
Suppose patient A has also expressed an overall preference for not doing high intensity exercise: 
$\rec[2] < \rec[8]$, ~ $\rec[3] < \rec[8]$, ~ $\rec[4] < \rec[8]$.

After discussing with patient A, the clinician elaborates a list of patient-centric goals, thus:
\begin{itemize}
\item $\goals = \{ \decrease~\pain, ~ \lpnot\increase~ \bloodpress, ~ \decrease~\fatigue, \break \lpnot\increase~\bodytemp \}$.
\end{itemize}

Note that not all properties (and effects) from recommendations need be included in $\goals$: 
for instance, \lymph\ is not concerned. 
The prioritisation of goals may be motivated by several criteria. 
Their specification is outside the scope of this work, but an example is the severity of a condition over the property it is associated with.
For patient A, the pain level from BC is a significant considerable obstacle to daily life, impeding normal routine. 
Additionally, the clinician is concerned with the patient's high blood pressure. 
Thus, $\decrease~\pain$ is the strictly most important goal, 
followed by $\lpnot\increase~\bloodpress$, 
which is followed by the equally important $\decrease~\fatigue$ and $\lpnot\increase~\bodytemp$ 
Then $\preccurlyeq$ over $\goals$ is defined as follows:  $\decrease~\fatigue \preccurlyeq \lpnot \increase~\bodytemp$;
$\lpnot \increase~\bodytemp \preccurlyeq \decrease~\fatigue$;
 $\decrease~\fatigue \prec \linebreak \lpnot\increase~\bloodpress \prec \decrease~\pain$; 
and visualised thus:
\begin{figure}[h]
\centering
\begin{tikzpicture}[sibling distance=10em,
  every node/.style = {align=center},level distance = 2em]

 \node {\decrease~\pain}
    child { node {\lpnot \increase~\bloodpress} 
        child { node {\decrease~\fatigue}}
        child { node {\lpnot \increase~\bodytemp}}
        };
\end{tikzpicture}
\end{figure}

The associated \abapg\ framework $\F_p = \abapgf$:
\begin{itemize}
\item $\A = \{ \rec[2], \rec[3], \rec[4], \rec[8] \}$, 
\item $\R = \{ \stdexercise \ot \rec[2], ~~ \lowexercise \ot \rec[3], \\
\lpnot\exercise \ot \rec[4], ~~ \lpnot\highexercise \ot \rec[8] \} \cup \{ \\
\increase~\bodytemp \ot \stdexercise, ~~ \\
\decrease~\fatigue \ot \stdexercise, ~~ \\
\decrease~\pain \ot \stdexercise, ~~ \\
\decrease~\fatigue \ot \lowexercise, ~~ \\
\decrease~\pain \ot \lowexercise, ~~ \\
\lpnot~\increase~\bloodpress \ot \lpnot \highexercise, ~~ \\
\lpnot~\increase~\bodytemp \ot \lpnot \exercise \} \cup \\
\{ \bloodpress \ot \top, ~ \high~\bodytemp \ot \top \} \cup \\
\{ \contr{\rec[4]} \ot \rec[2], int_{2, 4}, ~~ \contr{\rec[2]} \ot \rec[4], int_{2, 4}, \high~\bodytemp, ~~ \\
\contr{\rec[4]} \ot \rec[3], int_{3, 4}, ~~ \contr{\rec[3]} \ot \rec[4], int_{3, 4}, \high~\bodytemp, ~~ \\
\contr{\rec[8]} \ot \rec[2], int_{2, 8}, ~~ \contr{\rec[2]} \ot \rec[8], int_{2, 8}, \bloodpress \} \cup \\
\{ int_{2, 4} \ot \top, ~~ int_{3, 4} \ot \top, ~~ int_{2, 8} \ot \top \},$
\item $\leqslant$, $\goals$ and $\preccurlyeq$ as above (and $\LL$ and $~\contrary$ as per convention).
\end{itemize}

All contradiction interactions are triggered, 
giving arguments: $\{ \rec[2] \} \vdash \contr{\rec[4]}$, $\{ \rec[2] \} \vdash \contr{\rec[8]}$, $\{ \rec[4] \} \vdash \contr{\rec[2]}$,  $\{ \rec[4] \} \vdash \contr{\rec[3]}$,  $\{ \rec[3] \} \vdash \contr{\rec[4]}$,  $\{ \rec[8] \} \vdash \contr{\rec[2]}$. 
These indicate which recommendations are contradicting which other ones. 
Then, patient preferences help to determine the `stronger' arguments and (non-) $<$-attacks: $\{ \rec[2] \} \pattacks \{ \rec[4] \}$; $\{ \rec[4] \} \pattacks \{ \rec[2] \}$; $\{ \rec[3] \} \pattacks \{ \rec[4] \}$; $\{ \rec[4] \} \pattacks \{ \rec[3] \}$; $\{ \rec[8] \} \pattacks \{ \rec[2] \}$; but $\{ \rec[2] \} \npattacks \{ \rec[8] \}$. 
We thus see that, in particular, $\rec[2]$ suggesting \stdexercise, 
contradicting $\rec[8]$ but being less preferred, 
does not stand as an argument against \highexercise\ suggested by $\rec[8]$.

\abap\ semantics then resolves the conflicts. 
Briefly, as $\rec[3]$ and $\rec[4]$ are mutually contradicting, but each non-interacting with $\rec[8]$, 
they can be followed alongside $\rec[8]$, 
which itself kicks out $\rec[2]$. 
Thus, 
$\{ \rec[3], \rec[8] \}$, $\{ \rec[4], \rec[8] \}$ are  $<$-preferred extensions. 
The former suggests \lowexercise\ and advises against \highexercise; 
the latter urges not to exercise at all.
The corresponding goal extensions:
\begin{itemize}
\item $\goals[{\{ \rec[3], \rec[8] \}}] = \{\decrease~\fatigue, ~ \lpnot\increase~\bloodpress, ~ \break \decrease~\pain \}$;
\item $\goals[{\{ \rec[4], \rec[8] \}}] = \{\decrease~\fatigue, ~ \lpnot\increase~\bodytemp,\break
\decrease~\pain \}$.
\end{itemize}

In both goal extensions $\goals[{\{ \rec[3], \rec[8] \}}]$ and $\goals[{\{ \rec[4], \rec[8] \}}]$ 
the most important goal $\decrease~\pain$ is fulfilled, 
and so is $\decrease~\fatigue$.
But only $\goals[{\{ \rec[3], \rec[8] \}}]$ fulfills the second most important goal $\lpnot\increase \break ~\bloodpress$, 
so it is strictly better: 
$\goals[{\{ \rec[4], \rec[8] \}}] \lhd_{\mathbb{G}} \goals[{\{ \rec[3], \rec[8] \}}]$. 
Consequently, $\goals[{\{ \rec[3], \rec[8] \}}]$ is the top goal extension. 
Accordingly, $\rec[3]$ \linebreak ($\lowexercise$) and $\rec[8]$ ($\lpnot \highexercise$) should be followed.
\end{example}

We showed how \abapg\ allows for assumption-based, goal-driven and preference-respecting reasoning with TMR recommendations and interactions, 
taking into account patient-specific conditions, preferences over recommendations and priorities over goals. 

\section{Related Work}
\label{sec:related}


Argumentation (with or without preferences) has been successfully applied in health care
(see e.g.~\cite{Longo:2016,Atkinson.et.al:2017} for overviews). 
For instance, in \cite{Hunter:Williams:2012}, evidence from clinical trials is manually extracted from guidelines and synthesised to form arguments for treatment superiority, with attacks among arguments with conflicting claims. 
Based on treatment outcome indicators and the importance of evidence, 
user-specified preferences over arguments and argumentation semantics 
are used to identify the acceptable arguments. 
The focus is determining superiority among treatments, not concerning guideline recommendations or conflict resolution among those. 
Instead, argument aggregation for reasoning with guidelines is used in \cite{Grando:Glasspool:Boxwala:2012}.
There, 
arguments correspond to statements in guidelines
and, for a single specified goal, confidence of arguments is aggregated to identify the acceptable arguments. 
The focus is enacting recommendations from a single guideline, rather than reasoning with conflicting recommendations from multiple guidelines.
Other works, e.g.~\cite{Fox.et.al:2006,Tolchinsky:Cortes:Modgil:Caballero:Lopez-Navidad:2006,Qassas:Fogli:Giacomin:Guida:2016}, integrate argumentation with preferences to help clinicians to construct, exchange and evaluate arguments for and against decisions, 
rather than to reason with guidelines. 

The recent CONSULT project \citep{Consult:2018} applies  argumentation to reason with guidelines and patient preferences for managing post-stroke patients. 
\citeauthor{Consult:2018} manually represent guidelines in first-order logic (FOL) and use argument schemes \citep{Walton:1996} to construct arguments. 
We believe that using FOL for guideline representation is ad-hoc, 
and instead use the well-established TMR model to represent guideline recommendations and identify their interactions, 
which we then map into \abapg. 
Further, \citeauthor{Consult:2018} use argumentation with preferences modelled as attacks on attacks \cite{Modgil:2009} to resolve conflicts among recommendations. 
We instead incorporate preferences directly in the construction of attacks in \abapg. 
Importantly, this aspect and the ability to model and reason with goals allows us to meet Ariadne principles of patient management. 
We leave formal comparison with \citep{Consult:2018} for future work.


\cite{Wilk.et.al:2017} is a recent non-argumentative approach to reasoning with interacting guidelines, patient conditions and preferences. 
There, guideline recommendations are represented as actionable graphs and mapped into first-order logic (FOL) rules, 
while patient conditions and preferences are represented as FOL revision operators. 
Reasoning (guideline mitigation) amounts to applying revision operators to account for patient-specific conditions and preferences, and then finding models of the resulting FOL theory. 
Our approach is different in both knowledge representation---TMR model is richer than the mitigation-specific FOL, 
and computation mechanism---model finding is undecidable as opposed to finding preferred extensions. 
We also believe argumentation-based reasoning to be more transparent, as one can inspect the arguments, attacks among them and their interplay with preferences, 
in contrast to interpreting workings and results of a FOL theorem prover. 

Other approaches to reasoning with guidelines 
(see \cite{Peleg:2013,Riano:Ortega:2017} for overviews) focus on execution of single guidelines, e.g.~\cite{Leonardi:et:al:2012,Shalom.et.al:2016}, 
or identification of incompatibilities among guidelines, 
e.g.~answer set programming is used in \cite{Spiotta:Terenziani:Dupre:2017} to check temporal conformance;
statistical preference learning is used in \cite{Tsopra:Lamy:Sedki:2018} to identify inconsistencies in antibiotic therapy guidelines.
Yet other works concern preference elicitation to facilitate shared (clinician-patient) decision making.
In particular, \citeauthor{Sacchi.et.al:2015} incorporate patients' preferences in terms of QALYs, utilities and costs into the shared decision making model \cite{Sacchi.et.al:2015}. 
In effect, they propose a framework that supports patient preference elicitation and integrates them with patient health record to feed into decision models (particularly, decision trees) so as to facilitate shared (clinician-patient) decision making.
This allows to better inform both the clinician and the patient about the alternatives, but does not afford automatic resolution of interacting (e.g.~conflicting) recommendations.
It would be interesting to see how this could inform knowledge representation in our approach.

Goal-driven argumentative decision making (possibly with preferences) has been explored,
e.g.~\cite{Dung:Thang:Toni:2008,Amgoud:Prade:2009,Muller:Hunter:2012,Zeng.et.al:2018-AAMAS}. 
The settings there do not apply to reasoning with guidelines. 
As for goal-driven argumentative decision-making, the approach of \cite{Amgoud:Prade:2009} concerns general multiple criteria decision making in argumentation with preferences via reasoning backwards from goals to arguments. 
A follow-up application-specific approach of \cite{Muller:Hunter:2012} affords goal-driven argumentative documentation, analysis and making of decisions. 
\abapg\ differs from these approaches particularly in the direction of reasoning---from arguments to goals, 
which is more similar to assumption-based reasoning with goals and preferences as in \cite{Dung:Thang:Toni:2008}; 
as well as in using preferences (over goals) to select among extensions, as in e.g.\ \cite{Amgoud:Vesic:2014,Wakaki:2014}. 
It would be interesting to investigate the formal relationships with all these works in the future. 

We note that an argumentative approach with context was recently proposed in \cite{Zeng.et.al:2018-AAMAS}, where context rules and primitives involving patient state properties are used to assert defeasibility of logical implications between decisions, attributes, and goals. 
Thus, context-sensitivity is an important and desirable property in both medical and argumentative settings, 
and we addressed it in this work.

\section{Conclusions and Future Work}
\label{sec:future}


We proposed \abap\ to reason with guidelines and patient context. 
We mapped guideline representation as TMR model to \abap, 
incorporated in \abap\ patient-specific conditions and preferences, 
and augmented \abap\ to \abapg\ so as to account for patient-centric goals and their importance. 
\abapg\ yields contradiction-free recommendations and associated achievable goals while respecting the context of the patient. Our approach meets Ariadne principles:  
\textit{interaction assessment} is ensured by Theorem \ref{thm:cf} (contradiction-free recommendations); 
\textit{prioritisation and patient preferences} is ensured by 
Definition \ref{defn:abapg semantics} (extensions fulfill the most important goals)
and Theorem \ref{thm:pref} (most preferred non-conflicting recommendations are followed); 
\textit{individualised management} is ensured collectively by the above and the use of the patient context in \abapg\ (Definition \ref{defn: abapg patient framework}).
To the best of our knowledge, our work is unique in establishing a relationship between features of argumentative reasoning and principles of patient management. 
This hints at the adequacy of structured argumentation for this type of task, 
which we believe is important given the difficulty (both in terms of time and resources) of large-scale practical evaluations in a real setting.

In addition to several future work directions mentioned in Section \ref{sec:related}, 
we will extend our work to other interaction types identified in \cite{Zamborlini.et.al:2017}: repetition, alternative, etc.
We will also aim to incorporate various numerical measures from TMR, such as belief strength. 
This may yield additional preferences, 
and we will study how multiple types of possibly conflicting preferences can be simultaneously integrated in \abapg. 
Preference elicitation is a vast problem by itself, and we will explore integration with the relevant works, e.g.~\cite{Sacchi.et.al:2015}.
Last but not least, argumentation is well-suited for explanations, see e.g.~\cite{Moulin.et.al:2002,Atkinson.et.al:2017}, 
and we will study both the well-established and novel \abap\ mechanisms to explain to the patient or the clinician how and why \abapg\ arrives at the final recommendations.

\subsubsection*{Acknowledgements}

Kristijonas \v Cyras was supported by the EPSRC project 
\textbf{EP/P029558/1} ROAD2H: 
\emph{Resource Optimisation, Argumentation, Decision Support and Knowledge Transfer to Create Value via Learning Health Systems.} 
Tiago Oliveira was supported by JSPS KAKENHI Grant Number \textbf{JP18K18115}.

\noindent
\textbf{Data access statement}: 
No new data was collected in the course of this research.

\bibliographystyle{ACM-Reference-Format}  
\balance
\bibliography{references}


\begin{thebibliography}{00}


\ifx \showCODEN    \undefined \def \showCODEN     #1{\unskip}     \fi
\ifx \showDOI      \undefined \def \showDOI       #1{#1}\fi
\ifx \showISBNx    \undefined \def \showISBNx     #1{\unskip}     \fi
\ifx \showISBNxiii \undefined \def \showISBNxiii  #1{\unskip}     \fi
\ifx \showISSN     \undefined \def \showISSN      #1{\unskip}     \fi
\ifx \showLCCN     \undefined \def \showLCCN      #1{\unskip}     \fi
\ifx \shownote     \undefined \def \shownote      #1{#1}          \fi
\ifx \showarticletitle \undefined \def \showarticletitle #1{#1}   \fi
\ifx \showURL      \undefined \def \showURL       {\relax}        \fi
\providecommand\bibfield[2]{#2}
\providecommand\bibinfo[2]{#2}
\providecommand\natexlab[1]{#1}
\providecommand\showeprint[2][]{arXiv:#2}

\bibitem[\protect\citeauthoryear{Amgoud and Prade}{Amgoud and Prade}{2009}]%
        {Amgoud:Prade:2009}
\bibfield{author}{\bibinfo{person}{Leila Amgoud} {and} \bibinfo{person}{Henri
  Prade}.} \bibinfo{year}{2009}\natexlab{}.
\newblock \showarticletitle{{Using arguments for making and explaining
  decisions}}.
\newblock \bibinfo{journal}{{\em Artificial Intelligence\/}}
  \bibinfo{volume}{173}, \bibinfo{number}{3-4} (\bibinfo{year}{2009}),
  \bibinfo{pages}{413--436}.
\newblock
\showISBNx{0004-3702}
\showISSN{00043702}
\showDOI{%
\url{https://doi.org/10.1016/j.artint.2008.11.006}}


\bibitem[\protect\citeauthoryear{Amgoud and Vesic}{Amgoud and Vesic}{2014}]%
        {Amgoud:Vesic:2014}
\bibfield{author}{\bibinfo{person}{Leila Amgoud} {and} \bibinfo{person}{Srdjan
  Vesic}.} \bibinfo{year}{2014}\natexlab{}.
\newblock \showarticletitle{{Rich Preference-Based Argumentation Frameworks}}.
\newblock \bibinfo{journal}{{\em International Journal of Approximate
  Reasoning\/}} \bibinfo{volume}{55}, \bibinfo{number}{2}
  (\bibinfo{year}{2014}), \bibinfo{pages}{585--606}.
\newblock
\showISSN{0888613X}
\showDOI{%
\url{https://doi.org/10.1016/j.ijar.2013.10.010}}


\bibitem[\protect\citeauthoryear{Atkinson, Baroni, Giacomin, Hunter, Prakken,
  Reed, Simari, Thimm, and Villata}{Atkinson et~al\mbox{.}}{2017}]%
        {Atkinson.et.al:2017}
\bibfield{author}{\bibinfo{person}{Katie Atkinson}, \bibinfo{person}{Pietro
  Baroni}, \bibinfo{person}{Massimiliano Giacomin}, \bibinfo{person}{Anthony
  Hunter}, \bibinfo{person}{Henry Prakken}, \bibinfo{person}{Chris Reed},
  \bibinfo{person}{Guillermo~Ricardo Simari}, \bibinfo{person}{Matthias Thimm},
  {and} \bibinfo{person}{Serena Villata}.} \bibinfo{year}{2017}\natexlab{}.
\newblock \showarticletitle{{Towards Artificial Argumentation}}.
\newblock \bibinfo{journal}{{\em AI Magazine\/}} \bibinfo{volume}{38},
  \bibinfo{number}{3} (\bibinfo{year}{2017}), \bibinfo{pages}{25--36}.
\newblock


\bibitem[\protect\citeauthoryear{Bao, {\v{C}}yras, and Toni}{Bao
  et~al\mbox{.}}{2017}]%
        {Bao:Cyras:Toni:2017}
\bibfield{author}{\bibinfo{person}{Ziyi Bao}, \bibinfo{person}{Kristijonas
  {\v{C}}yras}, {and} \bibinfo{person}{Francesca Toni}.}
  \bibinfo{year}{2017}\natexlab{}.
\newblock \showarticletitle{{ABAplus: Attack Reversal in Abstract and
  Structured Argumentation with Preferences}}. In \bibinfo{booktitle}{{\em
  PRIMA 2017: Principles and Practice of Multi-Agent Systems - 20th
  International Conference}} (Nice), \bibfield{editor}{\bibinfo{person}{Bo~An},
  \bibinfo{person}{Ana L.~C. Bazzan}, \bibinfo{person}{Jo{\~{a}}o Leite},
  \bibinfo{person}{Serena Villata}, {and} \bibinfo{person}{Leendert van~der
  Torre}} (Eds.). \bibinfo{publisher}{Springer}, \bibinfo{pages}{420--437}.
\newblock
\showDOI{%
\url{https://doi.org/10.1007/978-3-319-69131-2_25}}


\bibitem[\protect\citeauthoryear{Bench-Capon, Atkinson, and
  McBurney}{Bench-Capon et~al\mbox{.}}{2012}]%
        {Bench-Capon:Atkinson:McBurney:2012}
\bibfield{author}{\bibinfo{person}{Trevor J~M Bench-Capon},
  \bibinfo{person}{Katie Atkinson}, {and} \bibinfo{person}{Peter McBurney}.}
  \bibinfo{year}{2012}\natexlab{}.
\newblock \showarticletitle{{Using argumentation to model agent decision making
  in economic experiments}}.
\newblock \bibinfo{journal}{{\em Autonomous Agents and Multi-Agent Systems\/}}
  \bibinfo{volume}{25}, \bibinfo{number}{1} (\bibinfo{year}{2012}),
  \bibinfo{pages}{183--208}.
\newblock
\showISSN{13872532}
\showDOI{%
\url{https://doi.org/10.1007/s10458-011-9173-6}}


\bibitem[\protect\citeauthoryear{Besnard, Garc{\'{i}}a, Hunter, Modgil,
  Prakken, Simari, and Toni}{Besnard et~al\mbox{.}}{2014}]%
        {Tutorials:2014}
\bibfield{author}{\bibinfo{person}{Philippe Besnard},
  \bibinfo{person}{Alejandro~Javier Garc{\'{i}}a}, \bibinfo{person}{Anthony
  Hunter}, \bibinfo{person}{Sanjay Modgil}, \bibinfo{person}{Henry Prakken},
  \bibinfo{person}{Guillermo~Ricardo Simari}, {and} \bibinfo{person}{Francesca
  Toni}.} \bibinfo{year}{2014}\natexlab{}.
\newblock \showarticletitle{{Introduction to Structured Argumentation}}.
\newblock \bibinfo{journal}{{\em Argument {\&} Computation\/}}
  \bibinfo{volume}{5}, \bibinfo{number}{1} (\bibinfo{year}{2014}),
  \bibinfo{pages}{1--4}.
\newblock
\showISSN{1946-2166}
\showDOI{%
\url{https://doi.org/10.1080/19462166.2013.869764}}


\bibitem[\protect\citeauthoryear{Bondarenko, Dung, Kowalski, and
  Toni}{Bondarenko et~al\mbox{.}}{1997}]%
        {Bondarenko:Dung:Kowalski:Toni:1997}
\bibfield{author}{\bibinfo{person}{Andrei Bondarenko},
  \bibinfo{person}{Phan~Minh Dung}, \bibinfo{person}{Robert Kowalski}, {and}
  \bibinfo{person}{Francesca Toni}.} \bibinfo{year}{1997}\natexlab{}.
\newblock \showarticletitle{{An Abstract, Argumentation-Theoretic Approach to
  Default Reasoning}}.
\newblock \bibinfo{journal}{{\em Artificial Intelligence\/}}
  \bibinfo{volume}{93}, \bibinfo{number}{97} (\bibinfo{year}{1997}),
  \bibinfo{pages}{63--101}.
\newblock
\showDOI{%
\url{https://doi.org/10.1016/S0004-3702(97)00015-5}}


\bibitem[\protect\citeauthoryear{{\v{C}}yras and Toni}{{\v{C}}yras and
  Toni}{2016}]%
        {Cyras:Toni:2016-KR}
\bibfield{author}{\bibinfo{person}{Kristijonas {\v{C}}yras} {and}
  \bibinfo{person}{Francesca Toni}.} \bibinfo{year}{2016}\natexlab{}.
\newblock \showarticletitle{{ABA+: Assumption-Based Argumentation with
  Preferences}}. In \bibinfo{booktitle}{{\em Principles of Knowledge
  Representation and Reasoning, 15th International Conference}},
  \bibfield{editor}{\bibinfo{person}{Chitta Baral}, \bibinfo{person}{James~P
  Delgrande}, {and} \bibinfo{person}{Frank Wolter}} (Eds.).
  \bibinfo{publisher}{AAAI Press}, \bibinfo{address}{Cape Town},
  \bibinfo{pages}{553--556}.
\newblock


\bibitem[\protect\citeauthoryear{Dimopoulos, Nebel, and Toni}{Dimopoulos
  et~al\mbox{.}}{2002}]%
        {Dimopoulos:Nebel:Toni:2002}
\bibfield{author}{\bibinfo{person}{Yannis Dimopoulos},
  \bibinfo{person}{Bernhard Nebel}, {and} \bibinfo{person}{Francesca Toni}.}
  \bibinfo{year}{2002}\natexlab{}.
\newblock \showarticletitle{{On The Computational Complexity of
  Assumption-Based Argumentation for Default Reasoning}}.
\newblock \bibinfo{journal}{{\em Artificial Intelligence\/}}
  \bibinfo{volume}{141}, \bibinfo{number}{1-2} (\bibinfo{year}{2002}),
  \bibinfo{pages}{57--78}.
\newblock
\showISSN{00043702}
\showDOI{%
\url{https://doi.org/10.1016/S0004-3702(02)00245-X}}


\bibitem[\protect\citeauthoryear{Dung, Thang, and Toni}{Dung
  et~al\mbox{.}}{2008}]%
        {Dung:Thang:Toni:2008}
\bibfield{author}{\bibinfo{person}{Phan~Minh Dung}, \bibinfo{person}{Phan~Minh
  Thang}, {and} \bibinfo{person}{Francesca Toni}.}
  \bibinfo{year}{2008}\natexlab{}.
\newblock \showarticletitle{{Towards Argumentation-Based Contract
  Negotiation}}. In \bibinfo{booktitle}{{\em Computational Models of Argument}}
  {\em (\bibinfo{series}{Frontiers in Artificial Intelligence and
  Applications})}, \bibfield{editor}{\bibinfo{person}{Philippe Besnard},
  \bibinfo{person}{Sylvie Doutre}, {and} \bibinfo{person}{Anthony Hunter}}
  (Eds.), Vol.~\bibinfo{volume}{172}. \bibinfo{publisher}{IOS Press},
  \bibinfo{address}{Toulouse}, \bibinfo{pages}{134--146}.
\newblock
\showISBNx{978-1-58603-859-5}


\bibitem[\protect\citeauthoryear{Fox, Black, Glasspool, Modgil, Oettinger,
  Patkar, and Williams}{Fox et~al\mbox{.}}{2006}]%
        {Fox.et.al:2006}
\bibfield{author}{\bibinfo{person}{John Fox}, \bibinfo{person}{Liz Black},
  \bibinfo{person}{David Glasspool}, \bibinfo{person}{Sanjay Modgil},
  \bibinfo{person}{Ayelet Oettinger}, \bibinfo{person}{Vivek Patkar}, {and}
  \bibinfo{person}{Matt Williams}.} \bibinfo{year}{2006}\natexlab{}.
\newblock \showarticletitle{{Towards a general model for argumentation
  services}}. In \bibinfo{booktitle}{{\em Argumentation for Consumers of
  Healthcare, Papers from the 2006 AAAI Spring Symposium}}.
  \bibinfo{publisher}{AAAI}, \bibinfo{address}{Stanford},
  \bibinfo{pages}{52--57}.
\newblock
\showDOI{%
\url{https://doi.org/10.1.1.148.5967}}


\bibitem[\protect\citeauthoryear{Fraccaro, {Arguello Casteleiro}, Ainsworth,
  and Buchan}{Fraccaro et~al\mbox{.}}{2015}]%
        {Fraccaro.et.al:2015}
\bibfield{author}{\bibinfo{person}{Paolo Fraccaro}, \bibinfo{person}{Mercedes
  {Arguello Casteleiro}}, \bibinfo{person}{John Ainsworth}, {and}
  \bibinfo{person}{Iain Buchan}.} \bibinfo{year}{2015}\natexlab{}.
\newblock \showarticletitle{{Adoption of Clinical Decision Support in
  Multimorbidity: A Systematic Review}}.
\newblock \bibinfo{journal}{{\em JMIR Medical Informatics\/}}
  \bibinfo{volume}{3}, \bibinfo{number}{1} (\bibinfo{year}{2015}),
  \bibinfo{pages}{e4}.
\newblock
\showISSN{2291-9694}
\showDOI{%
\url{https://doi.org/10.2196/medinform.3503}}


\bibitem[\protect\citeauthoryear{Grace, Mahony, O'donoghue, Heffernan, Molony,
  and Carroll}{Grace et~al\mbox{.}}{2013}]%
        {Grace.et.al:2013}
\bibfield{author}{\bibinfo{person}{Audrey Grace}, \bibinfo{person}{Carolanne
  Mahony}, \bibinfo{person}{John O'donoghue}, \bibinfo{person}{Tony Heffernan},
  \bibinfo{person}{David Molony}, {and} \bibinfo{person}{Thomas Carroll}.}
  \bibinfo{year}{2013}\natexlab{}.
\newblock \showarticletitle{{A vision for enhancing multimorbid care using
  clinical decision support systems}}.
\newblock \bibinfo{journal}{{\em Studies in Health Technology and
  Informatics\/}} \bibinfo{volume}{192}, \bibinfo{number}{1-2}
  (\bibinfo{year}{2013}), \bibinfo{pages}{1117}.
\newblock
\showISBNx{9781614992882}
\showISSN{09269630}
\showDOI{%
\url{https://doi.org/10.3233/978-1-61499-289-9-1117}}


\bibitem[\protect\citeauthoryear{Grando, Glasspool, and Boxwala}{Grando
  et~al\mbox{.}}{2012}]%
        {Grando:Glasspool:Boxwala:2012}
\bibfield{author}{\bibinfo{person}{Mar{\'{i}}a~Adela Grando},
  \bibinfo{person}{David Glasspool}, {and} \bibinfo{person}{Aziz~A. Boxwala}.}
  \bibinfo{year}{2012}\natexlab{}.
\newblock \showarticletitle{{Argumentation logic for the flexible enactment of
  goal-based medical guidelines}}.
\newblock \bibinfo{journal}{{\em Journal of Biomedical Informatics\/}}
  \bibinfo{volume}{45}, \bibinfo{number}{5} (\bibinfo{year}{2012}),
  \bibinfo{pages}{938--949}.
\newblock
\showISBNx{1532-0480 (Electronic) 1532-0464 (Linking)}
\showISSN{15320464}
\showDOI{%
\url{https://doi.org/10.1016/j.jbi.2012.03.005}}


\bibitem[\protect\citeauthoryear{Hunter and Williams}{Hunter and
  Williams}{2012}]%
        {Hunter:Williams:2012}
\bibfield{author}{\bibinfo{person}{Anthony Hunter} {and}
  \bibinfo{person}{Matthew Williams}.} \bibinfo{year}{2012}\natexlab{}.
\newblock \showarticletitle{{Aggregating evidence about the positive and
  negative effects of treatments}}.
\newblock \bibinfo{journal}{{\em Artificial Intelligence in Medicine\/}}
  \bibinfo{volume}{56}, \bibinfo{number}{3} (\bibinfo{year}{2012}),
  \bibinfo{pages}{173--190}.
\newblock
\showISBNx{0933-3657}
\showISSN{09333657}
\showDOI{%
\url{https://doi.org/10.1016/j.artmed.2012.09.004}}


\bibitem[\protect\citeauthoryear{Kaci and Patel}{Kaci and Patel}{2014}]%
        {Kaci:Patel:2014}
\bibfield{author}{\bibinfo{person}{Souhila Kaci} {and} \bibinfo{person}{Namrata
  Patel}.} \bibinfo{year}{2014}\natexlab{}.
\newblock \showarticletitle{{A postulate-based analysis of comparative
  preference statements}}.
\newblock \bibinfo{journal}{{\em Journal of Applied Logic\/}}
  \bibinfo{volume}{12}, \bibinfo{number}{4} (\bibinfo{year}{2014}),
  \bibinfo{pages}{501--521}.
\newblock
\showISBNx{9781577355588}
\showISSN{15708683}
\showDOI{%
\url{https://doi.org/10.1016/j.jal.2014.07.004}}


\bibitem[\protect\citeauthoryear{Kakas and Moraitis}{Kakas and
  Moraitis}{2003}]%
        {Kakas:Moraitis:2003}
\bibfield{author}{\bibinfo{person}{Antonis~C Kakas} {and}
  \bibinfo{person}{Pavlos Moraitis}.} \bibinfo{year}{2003}\natexlab{}.
\newblock \showarticletitle{{Argumentation Based Decision Making for Autonomous
  Agents}}. In \bibinfo{booktitle}{{\em 2nd International Joint Conference on
  Autonomous Agents {\&} Multiagent Systems}}. \bibinfo{publisher}{ACM Press},
  \bibinfo{address}{Melbourne}, \bibinfo{pages}{883--890}.
\newblock
\showDOI{%
\url{https://doi.org/10.1145/860575.860717}}


\bibitem[\protect\citeauthoryear{Kokciyan, Sassoon, Young, Chapman, Porat,
  Ashworth, Curcin, Modgil, Parsons, and Sklar}{Kokciyan et~al\mbox{.}}{2018}]%
        {Consult:2018}
\bibfield{author}{\bibinfo{person}{Nadin Kokciyan}, \bibinfo{person}{Isabel
  Sassoon}, \bibinfo{person}{Anthony Young}, \bibinfo{person}{Martin Chapman},
  \bibinfo{person}{Talya Porat}, \bibinfo{person}{Mark Ashworth},
  \bibinfo{person}{Vasa Curcin}, \bibinfo{person}{Sanjay Modgil},
  \bibinfo{person}{Simon Parsons}, {and} \bibinfo{person}{Elizabeth Sklar}.}
  \bibinfo{year}{2018}\natexlab{}.
\newblock \showarticletitle{{Towards an Argumentation System for Supporting
  Patients in Self-Managing their Chronic Conditions}}. In
  \bibinfo{booktitle}{{\em Joint Workshop on Health Intelligence (W3PHIAI)}}.
  \bibinfo{address}{New Orleans, Louisiana}.
\newblock


\bibitem[\protect\citeauthoryear{Lehtonen, Wallner, and
  J{\"{a}}rvisalo}{Lehtonen et~al\mbox{.}}{2019}]%
        {Lehtonen.et.al:2019-AAAI}
\bibfield{author}{\bibinfo{person}{Tuomo Lehtonen},
  \bibinfo{person}{Johannes~Peter Wallner}, {and} \bibinfo{person}{Matti
  J{\"{a}}rvisalo}.} \bibinfo{year}{2019}\natexlab{}.
\newblock \showarticletitle{{Reasoning over Assumption-Based Argumentation
  Frameworks via Direct Answer Set Programming Encodings}}. In
  \bibinfo{booktitle}{{\em 33rd AAAI Conference on Artificial Intelligence}}.
  \bibinfo{publisher}{AAAI Press}, \bibinfo{address}{Honolulu, Hawaii}.
\newblock


\bibitem[\protect\citeauthoryear{Leonardi, Bottrighi, Galliani, Terenziani,
  Messina, and {Della Corte}}{Leonardi et~al\mbox{.}}{2012}]%
        {Leonardi:et:al:2012}
\bibfield{author}{\bibinfo{person}{Giorgio Leonardi}, \bibinfo{person}{Alessio
  Bottrighi}, \bibinfo{person}{Gabriele Galliani}, \bibinfo{person}{Paolo
  Terenziani}, \bibinfo{person}{Antonio Messina}, {and}
  \bibinfo{person}{Francesco {Della Corte}}.} \bibinfo{year}{2012}\natexlab{}.
\newblock \showarticletitle{{Exceptions handling within GLARE clinical
  guideline framework.}}. In \bibinfo{booktitle}{{\em AMIA Annual Symposium
  Proceedings}}, Vol.~\bibinfo{volume}{2012}. \bibinfo{pages}{512--21}.
\newblock
\showISSN{1942-597X}


\bibitem[\protect\citeauthoryear{Longo}{Longo}{2016}]%
        {Longo:2016}
\bibfield{author}{\bibinfo{person}{Luca Longo}.}
  \bibinfo{year}{2016}\natexlab{}.
\newblock \showarticletitle{{Argumentation for Knowledge Representation,
  Conflict Resolution, Defeasible Inference and Its Integration with Machine
  Learning}}.
\newblock In \bibinfo{booktitle}{{\em Machine Learning for Health Informatics -
  State-of-the-Art and Future Challenges}},
  \bibfield{editor}{\bibinfo{person}{Andreas Holzinger}} (Ed.).
  Vol.~\bibinfo{volume}{9605}. \bibinfo{publisher}{Springer},
  \bibinfo{pages}{183--208}.
\newblock
\showISBNx{978-3-319-50477-3}
\showDOI{%
\url{https://doi.org/10.1007/978-3-319-50478-0_9}}


\bibitem[\protect\citeauthoryear{Modgil}{Modgil}{2009}]%
        {Modgil:2009}
\bibfield{author}{\bibinfo{person}{Sanjay Modgil}.}
  \bibinfo{year}{2009}\natexlab{}.
\newblock \showarticletitle{{Reasoning About Preferences in Argumentation
  Frameworks}}.
\newblock \bibinfo{journal}{{\em Artificial Intelligence\/}}
  \bibinfo{volume}{173}, \bibinfo{number}{9-10} (\bibinfo{year}{2009}),
  \bibinfo{pages}{901--934}.
\newblock
\showISSN{00043702}
\showDOI{%
\url{https://doi.org/10.1016/j.artint.2009.02.001}}


\bibitem[\protect\citeauthoryear{Moulin, Irandoust, B{\'{e}}langer, and
  Desbordes}{Moulin et~al\mbox{.}}{2002}]%
        {Moulin.et.al:2002}
\bibfield{author}{\bibinfo{person}{Bernard Moulin}, \bibinfo{person}{Hengameh
  Irandoust}, \bibinfo{person}{Micheline B{\'{e}}langer}, {and}
  \bibinfo{person}{G Desbordes}.} \bibinfo{year}{2002}\natexlab{}.
\newblock \showarticletitle{{Explanation and argumentation capabilities:
  Towards the creation of more persuasive agents}}.
\newblock \bibinfo{journal}{{\em Artificial Intelligence Review\/}}
  \bibinfo{volume}{17}, \bibinfo{number}{3} (\bibinfo{year}{2002}),
  \bibinfo{pages}{169--222}.
\newblock
\showISBNx{0269-2821}
\showISSN{02692821}
\showDOI{%
\url{https://doi.org/10.1023/A:1015023512975}}


\bibitem[\protect\citeauthoryear{Muller and Hunter}{Muller and Hunter}{2012}]%
        {Muller:Hunter:2012}
\bibfield{author}{\bibinfo{person}{Jann Muller} {and} \bibinfo{person}{Anthony
  Hunter}.} \bibinfo{year}{2012}\natexlab{}.
\newblock \showarticletitle{An Argumentation-Based Approach for Decision
  Making}. In \bibinfo{booktitle}{{\em 2012 IEEE 24th International Conference
  on Tools with Artificial Intelligence}}. \bibinfo{publisher}{IEEE},
  \bibinfo{address}{Washington, DC}, \bibinfo{pages}{564--571}.
\newblock
\showISBNx{978-0-7695-4915-6}
\showDOI{%
\url{https://doi.org/10.1109/ICTAI.2012.82}}


\bibitem[\protect\citeauthoryear{Muth, van~den Akker, Blom, Mallen, Rochon,
  Schellevis, Becker, Beyer, Gensichen, Kirchner, Perera, Prados-Torres,
  Scherer, Thiem, van~den Bussche, and Glasziou}{Muth et~al\mbox{.}}{2014}]%
        {Muth.et.al:2014}
\bibfield{author}{\bibinfo{person}{Christiane Muth}, \bibinfo{person}{Marjan
  van~den Akker}, \bibinfo{person}{Jeanet~W. Blom},
  \bibinfo{person}{Christian~D. Mallen}, \bibinfo{person}{Justine Rochon},
  \bibinfo{person}{Fran{\c{c}}ois~G. Schellevis}, \bibinfo{person}{Annette
  Becker}, \bibinfo{person}{Martin Beyer}, \bibinfo{person}{Jochen Gensichen},
  \bibinfo{person}{Hanna Kirchner}, \bibinfo{person}{Rafael Perera},
  \bibinfo{person}{Alexandra Prados-Torres}, \bibinfo{person}{Martin Scherer},
  \bibinfo{person}{Ulrich Thiem}, \bibinfo{person}{Hendrik van~den Bussche},
  {and} \bibinfo{person}{Paul~P. Glasziou}.} \bibinfo{year}{2014}\natexlab{}.
\newblock \showarticletitle{{The Ariadne principles: How to handle
  multimorbidity in primary care consultations}}.
\newblock \bibinfo{journal}{{\em BMC Medicine\/}} \bibinfo{volume}{12},
  \bibinfo{number}{1} (\bibinfo{year}{2014}), \bibinfo{pages}{1--11}.
\newblock
\showISBNx{1741-7015}
\showISSN{17417015}
\showDOI{%
\url{https://doi.org/10.1186/s12916-014-0223-1}}


\bibitem[\protect\citeauthoryear{Oliveira, Dauphin, Satoh, Tsumoto, and
  Novais}{Oliveira et~al\mbox{.}}{2018}]%
        {Oliveira:et.al:2018}
\bibfield{author}{\bibinfo{person}{Tiago Oliveira},
  \bibinfo{person}{J{\'{e}}r{\'{e}}mie Dauphin}, \bibinfo{person}{Ken Satoh},
  \bibinfo{person}{Shusaku Tsumoto}, {and} \bibinfo{person}{Paulo Novais}.}
  \bibinfo{year}{2018}\natexlab{}.
\newblock \showarticletitle{{Argumentation with Goals for Clinical Decision
  Support in Multimorbidity}}. In \bibinfo{booktitle}{{\em 17th International
  Conference on Autonomous Agents and Multiagent Systems}}.
  \bibinfo{publisher}{IFAAMAS}, \bibinfo{address}{Stockholm},
  \bibinfo{pages}{2031--2033}.
\newblock


\bibitem[\protect\citeauthoryear{Parsons, Sierra, and Jennings}{Parsons
  et~al\mbox{.}}{1998}]%
        {Parsons:Sierra:Jennings:1998}
\bibfield{author}{\bibinfo{person}{Simon Parsons}, \bibinfo{person}{Carles
  Sierra}, {and} \bibinfo{person}{Nick Jennings}.}
  \bibinfo{year}{1998}\natexlab{}.
\newblock \showarticletitle{{Agents that reason and negotiate by arguing}}.
\newblock \bibinfo{journal}{{\em Journal of Logic and Computation\/}}
  \bibinfo{volume}{8}, \bibinfo{number}{3} (\bibinfo{year}{1998}),
  \bibinfo{pages}{261--292}.
\newblock
\showISBNx{0955-792X}
\showISSN{0955792X}
\showDOI{%
\url{https://doi.org/10.1093/logcom/8.3.261}}


\bibitem[\protect\citeauthoryear{Peleg}{Peleg}{2013}]%
        {Peleg:2013}
\bibfield{author}{\bibinfo{person}{Mor Peleg}.}
  \bibinfo{year}{2013}\natexlab{}.
\newblock \showarticletitle{{Computer-interpretable clinical guidelines: A
  methodological review}}.
\newblock \bibinfo{journal}{{\em Journal of Biomedical Informatics\/}}
  \bibinfo{volume}{46}, \bibinfo{number}{4} (\bibinfo{year}{2013}),
  \bibinfo{pages}{744--763}.
\newblock
\showISBNx{doi:10.1016/j.jbi.2013.06.009}
\showISSN{15320464}
\showDOI{%
\url{https://doi.org/10.1016/j.jbi.2013.06.009}}


\bibitem[\protect\citeauthoryear{Qassas, Fogli, Giacomin, and Guida}{Qassas
  et~al\mbox{.}}{2016}]%
        {Qassas:Fogli:Giacomin:Guida:2016}
\bibfield{author}{\bibinfo{person}{Malik~Al Qassas}, \bibinfo{person}{Daniela
  Fogli}, \bibinfo{person}{Massimiliano Giacomin}, {and}
  \bibinfo{person}{Giovanni Guida}.} \bibinfo{year}{2016}\natexlab{}.
\newblock \showarticletitle{{ArgMed: A Support System for Medical Decision
  Making Based on the Analysis of Clinical Discussions}}.
\newblock In \bibinfo{booktitle}{{\em Real-World Decision Support Systems: Case
  Studies}}, \bibfield{editor}{\bibinfo{person}{Jason Papathanasiou},
  \bibinfo{person}{Nikolaos Ploskas}, {and} \bibinfo{person}{Isabelle Linden}}
  (Eds.). \bibinfo{publisher}{Springer}, \bibinfo{pages}{15--41}.
\newblock
\showDOI{%
\url{https://doi.org/10.1007/978-3-319-43916-7_2}}


\bibitem[\protect\citeauthoryear{Rahwan and Simari}{Rahwan and Simari}{2009}]%
        {Rahwan:Simari:2009}
\bibfield{author}{\bibinfo{person}{Iyad Rahwan} {and}
  \bibinfo{person}{Guillermo~Ricardo Simari}.} \bibinfo{year}{2009}\natexlab{}.
\newblock \bibinfo{booktitle}{{\em {Argumentation in Artificial
  Intelligence}}}.
\newblock \bibinfo{publisher}{Springer}.
\newblock
\showDOI{%
\url{https://doi.org/10.1007/978-0-387-98197-0}}


\bibitem[\protect\citeauthoryear{Ria{\~{n}}o and Ortega}{Ria{\~{n}}o and
  Ortega}{2017}]%
        {Riano:Ortega:2017}
\bibfield{author}{\bibinfo{person}{David Ria{\~{n}}o} {and}
  \bibinfo{person}{Wilfrido Ortega}.} \bibinfo{year}{2017}\natexlab{}.
\newblock \showarticletitle{{Computer technologies to integrate medical
  treatments to manage multimorbidity}}.
\newblock \bibinfo{journal}{{\em Journal of Biomedical Informatics\/}}
  \bibinfo{volume}{75}, \bibinfo{number}{September} (\bibinfo{year}{2017}),
  \bibinfo{pages}{1--13}.
\newblock
\showISBNx{1532-0480 (Electronic) 1532-0464 (Linking)}
\showISSN{15320464}
\showDOI{%
\url{https://doi.org/10.1016/j.jbi.2017.09.009}}


\bibitem[\protect\citeauthoryear{Sacchi, Rubrichi, Rognoni, Panzarasa,
  Parimbelli, Mazzanti, Napolitano, Priori, and Quaglini}{Sacchi
  et~al\mbox{.}}{2015}]%
        {Sacchi.et.al:2015}
\bibfield{author}{\bibinfo{person}{Lucia Sacchi}, \bibinfo{person}{Stefania
  Rubrichi}, \bibinfo{person}{Carla Rognoni}, \bibinfo{person}{Silvia
  Panzarasa}, \bibinfo{person}{Enea Parimbelli}, \bibinfo{person}{Andrea
  Mazzanti}, \bibinfo{person}{Carlo Napolitano}, \bibinfo{person}{Silvia~G.
  Priori}, {and} \bibinfo{person}{Silvana Quaglini}.}
  \bibinfo{year}{2015}\natexlab{}.
\newblock \showarticletitle{{From decision to shared-decision: Introducing
  patients' preferences into clinical decision analysis}}.
\newblock \bibinfo{journal}{{\em Artificial Intelligence in Medicine\/}}
  \bibinfo{volume}{65}, \bibinfo{number}{1} (\bibinfo{year}{2015}),
  \bibinfo{pages}{19--28}.
\newblock
\showISBNx{0933-3657}
\showISSN{18732860}
\showDOI{%
\url{https://doi.org/10.1016/j.artmed.2014.10.004}}


\bibitem[\protect\citeauthoryear{Shalom, Shahar, and Lunenfeld}{Shalom
  et~al\mbox{.}}{2016}]%
        {Shalom.et.al:2016}
\bibfield{author}{\bibinfo{person}{Erez Shalom}, \bibinfo{person}{Yuval
  Shahar}, {and} \bibinfo{person}{Eitan Lunenfeld}.}
  \bibinfo{year}{2016}\natexlab{}.
\newblock \showarticletitle{{An architecture for a continuous, user-driven, and
  data-driven application of clinical guidelines and its evaluation}}.
\newblock \bibinfo{journal}{{\em Journal of Biomedical Informatics\/}}
  \bibinfo{volume}{59}, \bibinfo{number}{November} (\bibinfo{year}{2016}),
  \bibinfo{pages}{130--148}.
\newblock
\showISBNx{1532-0480 (Electronic) 1532-0464 (Linking)}
\showISSN{15320464}
\showDOI{%
\url{https://doi.org/10.1016/j.jbi.2015.11.006}}


\bibitem[\protect\citeauthoryear{Spiotta, Terenziani, and Dupr{\'{e}}}{Spiotta
  et~al\mbox{.}}{2017}]%
        {Spiotta:Terenziani:Dupre:2017}
\bibfield{author}{\bibinfo{person}{Matteo Spiotta}, \bibinfo{person}{Paolo
  Terenziani}, {and} \bibinfo{person}{Daniele~Theseider Dupr{\'{e}}}.}
  \bibinfo{year}{2017}\natexlab{}.
\newblock \showarticletitle{{Temporal Conformance Analysis and Explanation of
  Clinical Guidelines Execution: An Answer Set Programming Approach}}.
\newblock \bibinfo{journal}{{\em IEEE Transactions on Knowledge and Data
  Engineering\/}} \bibinfo{volume}{29}, \bibinfo{number}{11}
  (\bibinfo{year}{2017}), \bibinfo{pages}{2567--2580}.
\newblock
\showISSN{10414347}
\showDOI{%
\url{https://doi.org/10.1109/TKDE.2017.2734084}}


\bibitem[\protect\citeauthoryear{Tolchinsky, Cort{\'{e}}s, Modgil, Caballero,
  and L{\'{o}}pez-Navidad}{Tolchinsky et~al\mbox{.}}{2006}]%
        {Tolchinsky:Cortes:Modgil:Caballero:Lopez-Navidad:2006}
\bibfield{author}{\bibinfo{person}{Pancho Tolchinsky}, \bibinfo{person}{Ulises
  Cort{\'{e}}s}, \bibinfo{person}{Sanjay Modgil}, \bibinfo{person}{Francisco
  Caballero}, {and} \bibinfo{person}{Antonio L{\'{o}}pez-Navidad}.}
  \bibinfo{year}{2006}\natexlab{}.
\newblock \showarticletitle{{Increasing Human-Organ Transplant Availability:
  Argumentation-Based Agent Deliberation}}.
\newblock \bibinfo{journal}{{\em IEEE Intelligent Systems\/}}
  \bibinfo{volume}{21}, \bibinfo{number}{6} (\bibinfo{year}{2006}),
  \bibinfo{pages}{30--37}.
\newblock
\showDOI{%
\url{https://doi.org/10.1109/MIS.2006.116}}


\bibitem[\protect\citeauthoryear{Tsopra, Lamy, and Sedki}{Tsopra
  et~al\mbox{.}}{2018}]%
        {Tsopra:Lamy:Sedki:2018}
\bibfield{author}{\bibinfo{person}{Rosy Tsopra}, \bibinfo{person}{Jean~Baptiste
  Lamy}, {and} \bibinfo{person}{Karima Sedki}.}
  \bibinfo{year}{2018}\natexlab{}.
\newblock \showarticletitle{{Using preference learning for detecting
  inconsistencies in clinical practice guidelines: Methods and application to
  antibiotherapy}}.
\newblock \bibinfo{journal}{{\em Artificial Intelligence in Medicine\/}}
  \bibinfo{volume}{89}, \bibinfo{number}{February} (\bibinfo{year}{2018}),
  \bibinfo{pages}{24--33}.
\newblock
\showISSN{18732860}
\showDOI{%
\url{https://doi.org/10.1016/j.artmed.2018.04.013}}


\bibitem[\protect\citeauthoryear{\v{C}yras and Oliveira}{\v{C}yras and
  Oliveira}{2018}]%
        {Cyras:Oliveira:2018}
\bibfield{author}{\bibinfo{person}{Kristijonas \v{C}yras} {and}
  \bibinfo{person}{Tiago Oliveira}.} \bibinfo{year}{2018}\natexlab{}.
\newblock \showarticletitle{Argumentation for Reasoning with Conflicting
  Clinical Guidelines and Preferences}. In \bibinfo{booktitle}{{\em Principles
  of Knowledge Representation and Reasoning, 16th International Conference}}
  (Tempe, AZ), \bibfield{editor}{\bibinfo{person}{Michael Thielscher},
  \bibinfo{person}{Francesca Toni}, {and} \bibinfo{person}{Frank Wolter}}
  (Eds.). \bibinfo{publisher}{{AAAI} Press}, \bibinfo{pages}{631--632}.
\newblock


\bibitem[\protect\citeauthoryear{Vermunt, Harmsen, Westert, {Olde Rikkert}, and
  Faber}{Vermunt et~al\mbox{.}}{2017}]%
        {Vermunt.et.al:2017}
\bibfield{author}{\bibinfo{person}{Neeltje P. C.~A. Vermunt},
  \bibinfo{person}{Mirjam Harmsen}, \bibinfo{person}{Gert~P. Westert},
  \bibinfo{person}{Marcel G.~M. {Olde Rikkert}}, {and}
  \bibinfo{person}{Marjan~J. Faber}.} \bibinfo{year}{2017}\natexlab{}.
\newblock \showarticletitle{{Collaborative goal setting with elderly patients
  with chronic disease or multimorbidity: a systematic review}}.
\newblock \bibinfo{journal}{{\em BMC Geriatrics\/}} \bibinfo{volume}{17},
  \bibinfo{number}{1} (\bibinfo{year}{2017}), \bibinfo{pages}{167}.
\newblock
\showISBNx{1471-2318}
\showISSN{1471-2318}
\showDOI{%
\url{https://doi.org/10.1186/s12877-017-0534-0}}


\bibitem[\protect\citeauthoryear{Wakaki}{Wakaki}{2014}]%
        {Wakaki:2014}
\bibfield{author}{\bibinfo{person}{Toshiko Wakaki}.}
  \bibinfo{year}{2014}\natexlab{}.
\newblock \showarticletitle{{Assumption-Based Argumentation Equipped with
  Preferences}}. In \bibinfo{booktitle}{{\em Principles and Practice of
  Multi-Agent Systems - 17th International Conference}},
  \bibfield{editor}{\bibinfo{person}{Hoa~Khanh Dam}, \bibinfo{person}{Jeremy~V
  Pitt}, \bibinfo{person}{Yang Xu}, \bibinfo{person}{Guido Governatori}, {and}
  \bibinfo{person}{Takayuki Ito}} (Eds.), Vol.~\bibinfo{volume}{8861}.
  \bibinfo{publisher}{Springer}, \bibinfo{address}{Gold Coast},
  \bibinfo{pages}{116--132}.
\newblock
\showDOI{%
\url{https://doi.org/10.1007/978-3-319-13191-7_10}}


\bibitem[\protect\citeauthoryear{Walton}{Walton}{1996}]%
        {Walton:1996}
\bibfield{author}{\bibinfo{person}{Douglas Walton}.}
  \bibinfo{year}{1996}\natexlab{}.
\newblock \bibinfo{booktitle}{{\em {Argumentation Schemes for Presumptive
  Reasoning}}}.
\newblock \bibinfo{publisher}{L. Erlbaum Associates}.
\newblock
\showISBNx{080582071X}


\bibitem[\protect\citeauthoryear{Wilk, Michalowski, Michalowski, Rosu, Carrier,
  and Kezadri-Hamiaz}{Wilk et~al\mbox{.}}{2017}]%
        {Wilk.et.al:2017}
\bibfield{author}{\bibinfo{person}{Szymon Wilk}, \bibinfo{person}{Martin
  Michalowski}, \bibinfo{person}{Wojtek Michalowski}, \bibinfo{person}{Daniela
  Rosu}, \bibinfo{person}{Marc Carrier}, {and} \bibinfo{person}{Mounira
  Kezadri-Hamiaz}.} \bibinfo{year}{2017}\natexlab{}.
\newblock \showarticletitle{{Comprehensive mitigation framework for concurrent
  application of multiple clinical practice guidelines}}.
\newblock \bibinfo{journal}{{\em Journal of Biomedical Informatics\/}}
  \bibinfo{volume}{66} (\bibinfo{year}{2017}), \bibinfo{pages}{52--71}.
\newblock
\showISBNx{1532-0464}
\showISSN{15320464}
\showDOI{%
\url{https://doi.org/10.1016/j.jbi.2016.12.002}}


\bibitem[\protect\citeauthoryear{Zamborlini, da~Silveira, Pruski, ten Teije,
  Geleijn, van~der Leeden, Stuiver, and van Harmelen}{Zamborlini
  et~al\mbox{.}}{2017}]%
        {Zamborlini.et.al:2017}
\bibfield{author}{\bibinfo{person}{Veruska Zamborlini}, \bibinfo{person}{Marcos
  da Silveira}, \bibinfo{person}{Cedric Pruski}, \bibinfo{person}{Annette ten
  Teije}, \bibinfo{person}{Edwin Geleijn}, \bibinfo{person}{Marike van~der
  Leeden}, \bibinfo{person}{Martijn Stuiver}, {and} \bibinfo{person}{Frank van
  Harmelen}.} \bibinfo{year}{2017}\natexlab{}.
\newblock \showarticletitle{{Analyzing interactions on combining multiple
  clinical guidelines}}.
\newblock \bibinfo{journal}{{\em Artificial Intelligence in Medicine\/}}
  \bibinfo{volume}{81} (\bibinfo{year}{2017}), \bibinfo{pages}{78--93}.
\newblock
\showISBNx{1873-2860 (Electronic) 0933-3657 (Linking)}
\showISSN{18732860}
\showDOI{%
\url{https://doi.org/10.1016/j.artmed.2017.03.012}}


\bibitem[\protect\citeauthoryear{Zeng, Fan, Miao, Leung, {Jing Jih}, and {Yew
  Soon}}{Zeng et~al\mbox{.}}{2018}]%
        {Zeng.et.al:2018-AAMAS}
\bibfield{author}{\bibinfo{person}{Zhiwei Zeng}, \bibinfo{person}{Xiuyi Fan},
  \bibinfo{person}{Chunyan Miao}, \bibinfo{person}{Cyril Leung},
  \bibinfo{person}{Chin {Jing Jih}}, {and} \bibinfo{person}{Ong {Yew Soon}}.}
  \bibinfo{year}{2018}\natexlab{}.
\newblock \showarticletitle{{Context-based and Explainable Decision Making with
  Argumentation}}. In \bibinfo{booktitle}{{\em 17th International Conference on
  Autonomous Agents and MultiAgent Systems}}. \bibinfo{publisher}{IFAAMAS},
  \bibinfo{address}{Stockholm}, \bibinfo{pages}{1114--1122}.
\newblock


\end{thebibliography}

\end{document}